\documentclass[12pt]{article}
\usepackage{amsmath,amssymb,amsthm}
\usepackage{times}
\usepackage{graphicx}
\usepackage{color}
\usepackage{multirow}
\usepackage{braket}
\usepackage[authoryear]{natbib}
\usepackage{algorithm}
\usepackage{algpseudocode}
\usepackage{diagbox} 
\usepackage{rotating}
\usepackage{bbm}
\usepackage{latexsym}

\newtheorem{definition}{Definition}
\newtheorem{theorem}{Theorem}
\newtheorem{lemma}{Lemma}

\newcommand{\KL}[2]{D_{\text{KL}}[#1||#2]}
\newcommand{\cP}{\mathcal{P}}
\newcommand{\cS}{\mathcal{S}}
\newcommand{\cT}{\mathcal{T}}
\newcommand{\cM}{\mathcal{M}}
\newcommand{\cL}{\mathcal{L}}
\newcommand{\vTheta}{\mathbf{\Theta}}
\newcommand{\vtheta}{\boldsymbol{\theta}}
\newcommand{\vxi}{\boldsymbol{\xi}}

\newcommand{\vrho}{\boldsymbol{\rho}}

\newcommand{\vEta}{\mathbf{H}}
\newcommand{\veta}{\boldsymbol{\eta}}
\newcommand{\vzeta}{\boldsymbol{\zeta}}
\newcommand{\vZeta}{\mathbf{Z}}
\newcommand{\va}{\mathbf{a}}
\newcommand{\vA}{\mathbf{A}}
\newcommand{\vg}{\mathbf{g}}
\newcommand{\vG}{\mathbf{G}}

\newcommand{\vK}{\mathbf{K}}
\newcommand{\vk}{\mathbf{k}}
\newcommand{\vf}{\mathbf{f}}
\newcommand{\vI}{\mathbf{I}}

\newcommand{\vt}{\mathbf{t}}

\newcommand{\vr}{\mathbf{r}}

\newcommand{\vs}{\mathbf{s}}

\newcommand{\vx}{\mathbf{x}}
\newcommand{\vy}{\mathbf{y}}
\newcommand{\vw}{\mathbf{w}}
\newcommand{\vW}{\mathbf{W}}

\newcommand{\vV}{\mathbf{V}}
\newcommand{\vu}{\mathbf{u}}
\newcommand{\vU}{\mathbf{U}}
\newcommand{\vz}{\mathbf{z}}
\newcommand{\vzero}{\mathbf{0}}

\newcommand{\vSigma}{\mathbf{\Sigma}}

\newcommand{\vmu}{\boldsymbol{\mu}}
\newcommand{\T}{^{\rm T}}
\newcommand{\inv}{^{-1}}

\newcommand{\argmin}{\mathop{\rm arg~min}\limits}
\newcommand{\coI}{_\mathrm{I}}
\newcommand{\coII}{_\mathrm{II}}

\usepackage{lipsum}

\usepackage{url}

\newcommand{\captionfonts}{\normalsize}

\makeatletter
\long\def\@makecaption#1#2{%
  \vskip\abovecaptionskip
  \sbox\@tempboxa{{\captionfonts #1: #2}}%
  \ifdim \wd\@tempboxa >\hsize
    {\captionfonts #1: #2\par}
  \else
    \hbox to\hsize{\hfil\box\@tempboxa\hfil}%
  \fi
  \vskip\belowcaptionskip}
\makeatother

\begin{document}
\hspace{13.9cm}1

\ \vspace{20mm}\\

{\LARGE Principal component analysis for Gaussian process posteriors}

\ \\
{\bf \large Hideaki Ishibashi$^{\displaystyle 1}$, Shotaro Akaho$^{\displaystyle 2}$}\\
{$^{\displaystyle 1}$Kyushu Institute of Technology.}\\
{$^{\displaystyle 2}$The National Institute of Advanced Industrial Science and Technology / RIKEN AIP.}\\
%



\thispagestyle{empty}
\markboth{}{NC instructions}
\ \vspace{-0mm}\\
%
\begin{center} {\bf Abstract} \end{center}
This paper proposes an extension of principal component analysis for Gaussian process (GP) posteriors, denoted by GP-PCA. Since GP-PCA estimates a low-dimensional space of GP posteriors, it can be used for meta-learning, which is a framework for improving the performance of target tasks by estimating a structure of a set of tasks. The issue is how to define a structure of a set of GPs with an infinite-dimensional parameter, such as coordinate system and a divergence. In this study, we reduce the infiniteness of GP to the finite-dimensional case under the information geometrical framework by considering a space of GP posteriors that have the same prior. In addition, we propose an approximation method of GP-PCA based on variational inference and demonstrate the effectiveness of GP-PCA as meta-learning through experiments.
\section{Introduction}

In recent years, machine learning algorithms have achieved high predictive accuracy when sufficient amounts of a dataset are accessible. On the other hand, since a sufficient dataset is not always available in practice, few-shot learning aiming to learn a model from an insufficient dataset attracts a lot of attention~\citep{Wang2020}. Meta-learning is a framework for achieving few-shot learning by estimating common knowledge among multiple tasks and applying that knowledge to target tasks when datasets of similar tasks are available~\citep{Vilalta2002,Hospedales2020,Huisman2021}. 

We focus on developing a meta-learning of Gaussian process (GP)~\citep{Rasmussen2005}. Conventional meta-learning methods of GP are classified into the covariance function-based approach and mean function-based approach. The former improves a predictive accuracy by estimating a covaiance function between tasks or a hyperprior of covariance functions associate with each task~\citep{Bonilla2007,Srijith2014}. Although the approach can estimate a suitable covariance function of each task considering a relationship between tasks, it assumes that the mean function of a prior is zero in general. Therefore, the approach may not be suitable for few-shot learning when this assumption does not hold. The mean function-based approach estimates a mean function and covariance function of a common prior across all tasks~\citep{Yu2005,Fortuin2020}. While this approach estimates the mean function, it requires a sufficient data when a task deviates from the prior distribution, since it does not estimate the task-specific knowledge. 

In our approach, we assume a set of GP posteriors lies on a low-dimensional subspace and estimate the subspace. Since the subspace parameterizes a mean function and covariance function associated with each task, we can estimate a suitable mean function and covariance function for each task. Especially, our approach can be interpreted as a kind of bagging of Gaussian processes (GPs)~\citep{Chen2009} when a rank of the subspace is zero, and expresses the relationships between tasks more flexibly by increasing the rank of the subspace. Therefore, even when only a few data are available, our approach can accurately estimate a GP posterior by projecting a posterior associated with a target task onto the subspace of a set of GP posteriors. Since the subspace is estimated by extending a principal component analysis (PCA) for GP posteriors, we call the proposed method GP-PCA.

For our method, we have to consider a space of GPs with an infinite-dimensional parameter. A structure of a probability space is nontrivial since Euclidean space is inappropriate as a structure of the space. For a finite parametric probability distribution, we can define a structure of its space using information geometry~\citep{Amari2016}. However, even if we use the information geometry, yet it remains difficult to define a space of GPs.

To overcome this problem, we consider defining the space of GP posteriors under the assumption that GP posteriors have the same prior. Then, we can show that the set of GP posteriors lies on a finite-dimensional subspace in an infinite-dimensional space of GP. By using this fact, we can reduce the task of GP-PCA to a task of estimating a subspace on finite-dimensional space. Additionally, we developed a fast approximation method for GP-PCA using a sparse GP based on variational inference and a hyperparameter optimization method based on hierarchical Bayes.

The remainder of the paper is organized as follows. In Section~2, we explain our approach and compare it with the conventional meta-learning of GP. In Section~3, we explain the information geometry and PCA for exponential families. In Section~4, we define a set of GP posteriors in terms of information geometry and show that the task can be reduced to a finite-dimensional case. In Section~5, we explain the algorithm of GP-PCA , its approximation method, and its hyperparameter optimization method. In Section~6, we demonstrate the effectiveness of the proposed method. Finally, Section~7 presents the conclusion.

\section{Conventional meta-learning of Gaussian process and our approach}
\subsection{Gaussian process (GP)}
Now let us review the GP regression.
First, we present the definition of notations. An output vector of function $f$ corresponding to input set $X=\{x_n\}^N_{n=1}$ is denoted by $\vf$ or $f(X)$. When an input set is denoted with a subscript, such as $X_A$, the corresponding output vector is also denoted with the subscript, such as $\vf_A$. Similarly, while a vector of kernel $k(x,x')$ between $X$ and $x$ is denoted by $\vk(x):=k(X,x)$, a gram matrix between $X$ and $X$ is denoted by $\vK:=k(X,X)$. The treatment of the subscript is the same as a function.

GP is a stochastic process with respect to a function $f:\mathcal{X} \rightarrow \mathbb{R}$. It is parameterized by the mean function $\mu:\mathcal{X} \rightarrow \mathbb{R}$ and covariance function $\sigma:\mathcal{X}\times\mathcal{X} \rightarrow \mathbb{R}$. The GP has a marginalization property. This means that a vector $\vf=f(X)$ corresponding to an arbitrary input set $X$ can consistently follow a multivariate normal distribution $\vf \sim \mathcal{N}(\vmu,\vSigma)$, where $\vmu:=\mu(X)$ and $\vSigma:=\sigma(X,X)$. Therefore, GP can be regarded as an infinite-dimensional multivariate normal distribution intuitively. 

GP regression is a posterior distribution for given samples and GP prior, as formulated below. Let $x \in \mathcal{X}$ and $y \in \mathbb{R}$ be an input vector and output, respectively. We assume that the relationship between $x \in \mathcal{X}$ and $y \in \mathbb{R}$ is denoted as $y = f(x) + \varepsilon$, where $\varepsilon$ is a noise. The task of regression is to estimate a function $f:\mathcal{X} \rightarrow \mathbb{R}$ from an input set $X=\{x_n\}^N_{n=1}$ and corresponding output vector $\vy=(y_1,y_2,\ldots,y_N)$. Here, we assume that $\vy$ is generated from Gaussian distribution with mean $\vf$ and variance $\beta\inv\vI$, and that the prior distribution is a GP with a mean function $\mu_0$ and covariance function $k$. For any $x_+$, $p(f(x_+),\vy)$ is obtained as
\begin{equation*}
  \left[
    \begin{array}{c}
      \vy  \\
      f(x_+)
    \end{array}
  \right]
  \sim \mathcal{N}\left(
  \left[
    \begin{array}{c}
      \vmu_0  \\
      \mu_0(x_+)
\end{array}
  \right]
  ,
  \left[
    \begin{array}{cc}
      \vK+\beta^{-1}\vI & \vk(x_+) \\
      \vk\T(x_+) & k(x_+,x_+)
    \end{array}
  \right]\right).
\end{equation*}
Since the posterior distribution is a conditional distribution of $f(x_+)$ given $\vy$, the mean and covariance function for a new input $x_+$ of the posterior distribution can be obtained by closed form as $p(f(x_+) \mid \vy) = \mathcal{N}(f(x_+) \mid \mu(x_+),\sigma(x_+,x_+))$, where
\begin{align}
    \mu(x_+)&=\mu_0(x_+)+\vk\T(x_+)\left(\vK+\beta\inv\vI\right)\inv(\vy-\vmu_0), \label{eq_gp_posterior_mean} \\
    \sigma(x_+,x_+)&=k(x_+,x_+)-\vk\T(x_+)\left(\vK+\beta\inv\vI\right)\inv\vk(x_+), \label{eq_gp_posterior_cov}
\end{align}
which means that the posterior is given by another GP.

We can also interpret that the posterior is obtained using Bayes' theorem. When $X$ and $\vy$ are observed, the posterior for $\vf=f(X)$ is derived as follows:
\begin{equation*}
    q(\vf\mid\vy) = \frac{p(\vy\mid\vf)p(\vf)}{p(\vy)}.
\end{equation*}
By using $q(\vf\mid\vy)$, the predictive distribution for new input data $x_+$ is described as follows:
\begin{equation*}
    q(f(x_+)\mid\vy) = \int p(f(x_+)\mid\vf)q(\vf\mid\vy)d\vf,
    \label{eq_int_pos}
\end{equation*}
where $p(f(x_+)\mid\vf)$ is a conditional prior. Letting $q(\vf\mid\vy)=\mathcal{N}(\vmu,\vSigma)$, $\vmu$ and $\vSigma$ are obtained as follows:
\begin{align*}
    \vmu&=\vmu_0+\vK\left(\vK+\beta\inv\vI\right)\inv(\vy-\vmu_0),\\
    \vSigma&=\vK-\vK\left(\vK+\beta\inv\vI\right)\inv\vK.
\end{align*}
Furthermore, the predictive distribution is derived as
\begin{align*}
    q(f(x_+)\mid\vy)=&\mathcal{N}(f(x_+)\mid \mu(x_+),\sigma(x_+,x_+)) \\
    \mu(x_+)&=\mu_0(x_+)+\vk\T(x_+)\vK\inv\vmu, \\
    \sigma(x_+,x_+)&=k(x_+,x_+)+\vk\T(x_+)\vK\inv(\vSigma-\vK)\vK\inv\vk(x_+).
\end{align*}
Since $p(f(x_+)\mid\vf)$ is a prior distribution, $q(f(x_+)\mid\vy)$ is determined uniquely when $p(\vf\mid\vy)$ is given. 

\subsection{Conventional meta-learning of GP}
Meta-learning is a framework that estimates a common knowledge of tasks through similar but different learning tasks and adapts this knowledge to target tasks~\citep{Vilalta2002,Hospedales2020,Huisman2021}. As a framework similar to meta-learning, there is a multi-task learning that improves the predictive accuracy of each task by estimating a common knowledge of tasks~\citep{Zhang2017}. Given that the approach of the meta-learning for GP is the same as that of the multi-task learning for GP, we explain the conventional meta-learning and multi-task learning methods.

The standard approach to achiving meta-learning of GP is to learns of a covariance function of a prior. This approach is classified into two approaches: the feature learning approach and the cross-covariance approach. In the feature learning approach, we consider selecting input features in each task by estimating hyperparameters of auto-relevant determination kernel or multi-kernel~\citep{Srijith2014,Titsias2011}. Then, by estimating a hyperprior of the kernels associated with tasks, knowledge common among tasks is shared. On the other hand, in the cross-covariance approach, relationships between tasks are modeled by the Kronecker product of a covariance function of samples and that of tasks~\citep{Bonilla2007}. That is, covariance between tasks is defined as
\begin{equation*}
    {\rm cov}[f_d(x),f_{d'}(x)] = \sum^R_{r=1}k^{(\rm task)}_{dd'}k^{(\rm input)}(x,x),
\end{equation*}
where $k^{(\rm task)}$ and $k^{(\rm input)}$ are covariance functions between tasks and inputs, respectively. In geostatistics, the model is called intrinsic coregionalized model (ICM)~\citep{Alvarez2011}. In recent years, the conbination of feature learning and cross-covariance approaches has been proposed~\citep{Li2018}. Although these approaches can estimate a suitable covariance function for each task, they do not estimate the mean function. Therefore, these approaches are not always suitable for few-shot learning setting.

An approach that learns a mean function has also been proposed. The simplest method is to estimate a prior by modeling as hierarchical Bayes~\citep{Schwaighofer2005, Yu2005}. The method estimates a mean function and covariance function of a GP prior by assuming a normal-inverse-Wishart distribution as a hyperprior of the GP prior. In recent years, the methods estimating a prior shared by all tasks by using deep neural networks have been proposed~\citep{Fortuin2020,Rothfuss2021}. These methods do not estimate task-specific knowledge, since they do not considered the relationships between tasks. Therefore, the more a predictive function associated with a task deviates from the prior distribution, the more the required training data.

\subsection{Our approach}

In our approach, assuming that a set of GP posteriors lies on a low-dimensional subspace, we estimate that subspace. Since the subspace parameterizes a mean function and covariance function associated with each task, by projecting a GP posterior estimated from the dataset to the subspace, we can estimate accurately task-specific GP posterior even when only a few data are available.

Our approach can be regarded as a generalization of bagging for GP~\citep{Chen2009} to a meta-learning setting. To illustrate this, consider the bagging in the case of GP. Bagging is a method of obtaining more robust and accurate models using bootstrap datasets re-sampled from the training data. The simplest bagging method for GP is to average a set of GP posteriors, each of which is estimated from each bootstrap dataset~\citep{Chen2009}. This corresponds to a case estimating a zero-dimensional subspace (i.e., a point) of GP posteriors in our approach. Namely, estimating a zero-dimensional subspace is equal to assuming that a dataset for all tasks is generated from the same distribution. In meta-learning setting, since a dataset is generated from a different distribution for each task, we estimate the task-specific GP posterior while sharing data between tasks by estimating low-dimensional subspaces.

Furthermore, our approach can use any prior distribution as long as all tasks share that prior distribution. Therefore, the proposed method can be combined with the conventional meta-learning of GP, such as hierarchical Bayes-based GP, which uses as a hyperparameter optimization method.

\section{Information geometry-based dimensionality reduction for exponential families}
Information geometry\citep{Amari2016} has enabled interpretation of many complicated machine learning algorithms from a unified viewpoint.
In this section, we review the information geometry of the exponential family that is necessary to formulate
the dimension reduction of GP. 

\subsection{Information geometry of the exponential family}
The exponential family is a distribution parameterized by $\vxi=(\xi_1,\xi_2,\cdots,\xi_D)$ as follows:
\begin{align*}
    p(x\mid \vxi)=\exp(\vxi\T \vG(x) + C(x) - \psi(\vxi)).
\end{align*}
In information geometry, a set of $p(x\mid \vxi)$ is regarded as a Riemannian manifold denoted by $\cS$, i.e.,
each distribution $p(x\mid\vxi)$ is a point specified by a local coordinate $\vxi$ on $\cS$.
The structure of the Riemannian manifold is determined by a metric and an (Affine) connection.
It is statistically natural to define a metric of $\cS$ by Fisher information 
\begin{equation*}
    g_{ij}(\vxi)=E_{p(x\mid \vxi)}\left[\left(\frac{\partial}{\partial \xi_i}\log{p(x\mid \vxi)}\right)\left(\frac{\partial}{\partial \xi_j}\log{p(x\mid \vxi)}\right)\right],
\end{equation*}
which defines a local linear structure of the space.
As a connection, it is natural to consider a family of connections parameterized by $\alpha\in\mathbb{R}$, which is called
$\alpha$-connection. In information geometry,
\textit{flatness} of a space is an important notion and many machine learning algorithms are interpreted as a procedure to obtain an orthogonal projection onto a flat subspace\citep{Amari2010,Amari2016,Nielsen2020}.
Unlike ordinary Euclidean space, the notion of flatness is not trivial.
In fact, we can define the flatness depending on the value of $\alpha$, and
further, it has been shown that in particular when $\alpha=\pm1$, $\cS$ can be regarded as a flat manifold with respect to a corresponding coordinate system (i.e., when $\alpha=1$, $\cS$ is a flat manifold defined in a coordinate system specified by natural parameter $\vxi$, which is called e-coordinate system, and this flatness is called e-flat). When $\alpha=-1$, $\cS$ is a flat manifold defined in a coordinate system specified by another parameter $\vzeta:=E_{p(x\mid \vxi)}[\vG(x)]$ (mean parameter), which is called m-coordinate system, and this flatness is called m-flat\footnote{\normalsize e- and m- are abbreviations of exponential and mixture, which come from corresponding families of distribution}. 

The e-coordinate and m-coordinate have an important property of duality.
There is a bijection between $\vxi$ and $\vzeta$, which can be described as Legendre transform. The following equation with respect to $\vxi$ and $\vzeta$ holds:
\begin{equation*}
    \psi(\vxi)+\phi(\vzeta)-\vxi\T\vzeta = 0,
\end{equation*}
where $\psi(\vxi)$ and $\phi(\vzeta)$ are potential functions of $\vxi$ and $\vzeta$, respectively. From the above equation, $\vxi$ and $\vzeta$ can be mutually transformed by Legendre transformation as follows.
\begin{equation}\label{legendre}
    \frac{\partial \psi(\vxi)}{\partial \vxi}=\vzeta, \qquad
    \frac{\partial \phi(\vzeta)}{\partial \vzeta}=\vxi. 
\end{equation}
Flat space is defined as a linear space spanned by corresponding coordinate system.
Since any point in the whole space $\cS$ can be represented by both e- and m- coordinates, $\cS$ is called a dually flat manifold.

In a dually flat manifold, we can consider two kinds of linear subspaces (submanifold): e-flat and m-flat subspaces. Let $\vxi_i$ and $\vzeta_i$ be an e-coordinate and m-coordinate of $p_i \in \cS$. While an e-flat subspace is defined as a linear combination of $\Xi=\{\vxi_i\}^I_{i=1}$, an m-flat subspace is defined as a linear combination of $Z=\{\vzeta_i\}^I_{i=1}$. Let $\cM_e$ and $\cM_m$ be e-flat and m-flat subspaces, respectively. Then, $\cM_e$ and $\cM_m$ are described as follows:
\begin{equation}
    \cM_e=\Set{ \vxi(\mathbf{t},\Xi) = \sum^M_{m=1}t_m\vxi_m | \sum^M_{m=1}t_m=1 },
\end{equation}
\begin{equation}
    \cM_m=\Set{ \vzeta(\mathbf{t},Z) = \sum^M_{m=1}t_m\vzeta_m | \sum^M_{m=1}t_m=1 },
\end{equation}
where $\vt=(t_1,t_2,\cdots,t_M)$. When $M=2$, $\cM_e$ and $\cM_m$ are called e-geodesic and m-geodesic, respectively\footnote{ \normalsize geodesic is an extended notion of a straight line}.
Note that e-flat subspace does not always become m-flat and vice versa, since $\vxi$ and $\vzeta$ are in a nonlinear relationship in general from (\ref{legendre}),  

By using $\vxi_i$ and $\vzeta_i$, we define a Kullback-Leibler (KL) divergence between the two points $p_i,p_j \in \cS$ as follows:
\begin{equation}
    \KL{p_i}{p_j}=\psi(\vxi_i)+\phi(\vzeta_j)-\vxi^T_i\vzeta_j.
\end{equation}
We denote the KL divergence by using e-coordinates or m-coordinates depending on the situation, i.e., $\KL{\vxi_i}{\vxi_j}$ and $\KL{\vzeta_i}{\vzeta_j}$. The following theorems show an interesting duality of e-coordinate and m-coordinate.
\begin{theorem}[Pythagorean theorem~\citep{Amari2016}]
Let $p_i$, $p_j$ and $p_k$ be points on $\cS$. If an e-geodesic between $p_i$ and $p_j$ and an m-geodesic between $p_j$ and $p_k$ are dually orthogonal, i.e., $(\vxi_i-\vxi_j)\T(\vzeta_j-\vzeta_k)=0$ holds\footnote{\normalsize This dually orthogonality (ordinary sum product between dual coordinate) coincides with the orthogonality in each coordinate with respect to Riemmanian metric $g_{ij}$}. 
Then, the following relationship holds.
\begin{equation*}
    \KL{p_i}{p_k} = \KL{p_i}{p_j} + \KL{p_j}{p_k}.
\end{equation*}
\end{theorem}

When an m-geodesic between $p \in \cS$ and $q \in \cM_e \subset \cS$ is orthogonal at $q$ to $\cM_e$, $q$ is called m-projection from $p$ to $\cM_e$. Similarly, when an e-geodesic between $p \in \cS$ and $q \in \cM_m \subset \cS$ is orthogonal at $q$ to $\cM_m$, $q$ is called e-projection from $p$ to $\cM_m$. From the Pythagorean theorem, the following theorem holds:
\begin{theorem}[Projection theorem\citep{Amari2016}]
An m-projection from $p \in \cS$ to $q \in \cM_e$ uniquely exists and it minimizes $\KL{p}{q}$. Similarly, an e-projection from $p \in \cS$ to $q \in \cM_m$ uniquely exists, and it minimizes $\KL{q}{p}$. 
\end{theorem}
This theorem tells us a lot of insights on machine learning algorithms. For a flat subspace, we can obtain unique projection, which is computationally convenient, for instance, we can avoid local optimum problem. 
As a projection for the flat subspace, we should
choose the dual projection (m-projection for e-flat subspace
and e-projection for m-flat subspace) for the uniqueness.
In fact, the m-projection onto m-flat subspace is not necessarily unique.
Further, the projection is characterized as minimizing KL divergence
that is a natural measure of distance between two distributions
in statistics.

Based on these properties, PCA for exponential families have been proposed by \citep{Collins2001,Akaho2004}. We explain the method below.

\subsection{PCA for exponential families}

Let $\cS$ be a set of exponential families. Since there are two types of subspaces on $\cS$: e-flat and m-flat subspaces, we can consider two PCAs for a dataset $\cP=\{p_1,p_2,\cdots,p_I\}\in \cS$. One is e-PCA, which estimates an e-flat affine subspace $\cM_e$. The other is m-PCA, which estimates an m-flat affine subspace $\cM_m$. 
There are two main reasons to use this information geometric PCA instead of ordinary PCA: The one is that the projection point given by the ordinary PCA
is not always well-defined (e.g., negative variance), while the e-PCA (m-PCA) gives a well-defined projection.
The other is that Euclidean distance implicitly assumed in ordinary PCA is not appropriate distance for distributions, while KL divergence used in e-PCA (m-PCA) is more natural.
In the following, 
although we only explain e-PCA, the same argument holds for m-PCA.

Assume that $\cM_e$ can be described in e-coordinate by $L$ basis $\{\vu_1, \vu_2, \ldots, \vu_L\}$ and offset $\vu_0\in\cS$, where $\vu_l \in \mathbb{R}^D, (l=0,1,\ldots,L)$. It means that using a weight vector $\vw=(w_1, w_2, \ldots, w_L)\T\in \mathbb{R}^L$ and basis $\vU=(\vu_0,\vu_1, \vu_2, \ldots, \vu_L)\T$, any point on $\cM_e$ can be represented as
\begin{align*}
    \vxi(\vw, \vU) &= \sum^L_{l=1} w_l \vu_l + \vu_0 \\
    &=(1,\vw\T) \vU.
\end{align*}

When $\{\vxi_i\}^I_{i=1}$ is obtained, the task of e-PCA is to estimate $\vW=(\vw_1,\vw_2,\ldots,\vw_I)\T$ and $\vU$ minimizing the following objective function.
\begin{align}
    E(\vW,\vU)=\sum^I_{i=1}D_{KL}[\vxi_i || \vxi(\vw_i, \vU)].
    \label{objective_function_epca}
\end{align}
Because $\vW$ and $\vU$ minimizing $E(\vW,\vU)$ cannot be obtained analytically in general, e-PCA alternatively estimates $\vW$ and $\vU$ using a gradient method. Let $\vzeta_i$ and $\tilde{\vzeta}_i$ be m-coordinates of $\vxi_i$ and $\vxi(\vw_i,\vU)$, respectively. We denote matrices of $\{\vzeta_i\}^I_{i=1}$ and $\{\tilde{\vzeta}_i\}^I_{i=1}$ by $\mathbf{Z}$ and $\tilde{\mathbf{Z}}$, respectively. The gradients of Eq.~\eqref{objective_function_epca} with respect to $\vW$ and $\vU$ are given by the following equations.
\begin{align}
    \frac{\partial E(\vW,\vU)}{\partial \vW}&=(\hat{\vZeta}-\vZeta)\tilde{\vU}\T, \\
    \frac{\partial E(\vW,\vU)}{\partial \vU}&=\vW\T(\hat{\vZeta}-\vZeta),
    \label{delivation_objective_function_epca}
\end{align}
where $\tilde{\vU}=(\vu_1, \vu_2, \ldots, \vu_L)\T$.

For multivariate normal distributions, each probabilistic distribution can be parameterized a mean vector $\vmu$ and covariance matrix $\vSigma$. Letting $G_1(\vx)=\vx$ and $G_2(\vx)= \vx \vx\T$, the e-coordinate $\vxi$ can be described as follows:
\begin{align}
    \vxi &= (\vtheta\T,{\rm vec}(\vTheta)\T)\T,
\end{align}
where $\vtheta=\vSigma\inv\vmu$, $\vTheta=-\frac{1}{2}\vSigma\inv$. On the other hand, the m-coordinate can be described as follows:
\begin{align}
    \vzeta=(\veta\T,{\rm vec}(\vEta)\T)\T,
\end{align}
where $\veta= \vmu$, $\vEta = \vmu \vmu\T + \vSigma$.
Then, the transform between $\vxi$ and $\vzeta$ can be described as follows:
\begin{align*}
    \vtheta &= \left(\vEta - \veta\veta\T \right)\inv\veta, \\
    \vTheta &= -\frac{1}{2}\left(\vEta - \veta\veta\T\right)\inv, \\
    \veta &= -\frac{1}{2}\vTheta\inv\vtheta, \\
    \vEta &= \frac{1}{4}\vTheta\inv\vtheta\vtheta\T\vTheta\inv -\frac{1}{2}\vTheta\inv.
\end{align*}

\section{PCA for Gaussian processes (GP-PCA)}
\begin{figure}
    \centering
    \includegraphics[width=15cm]{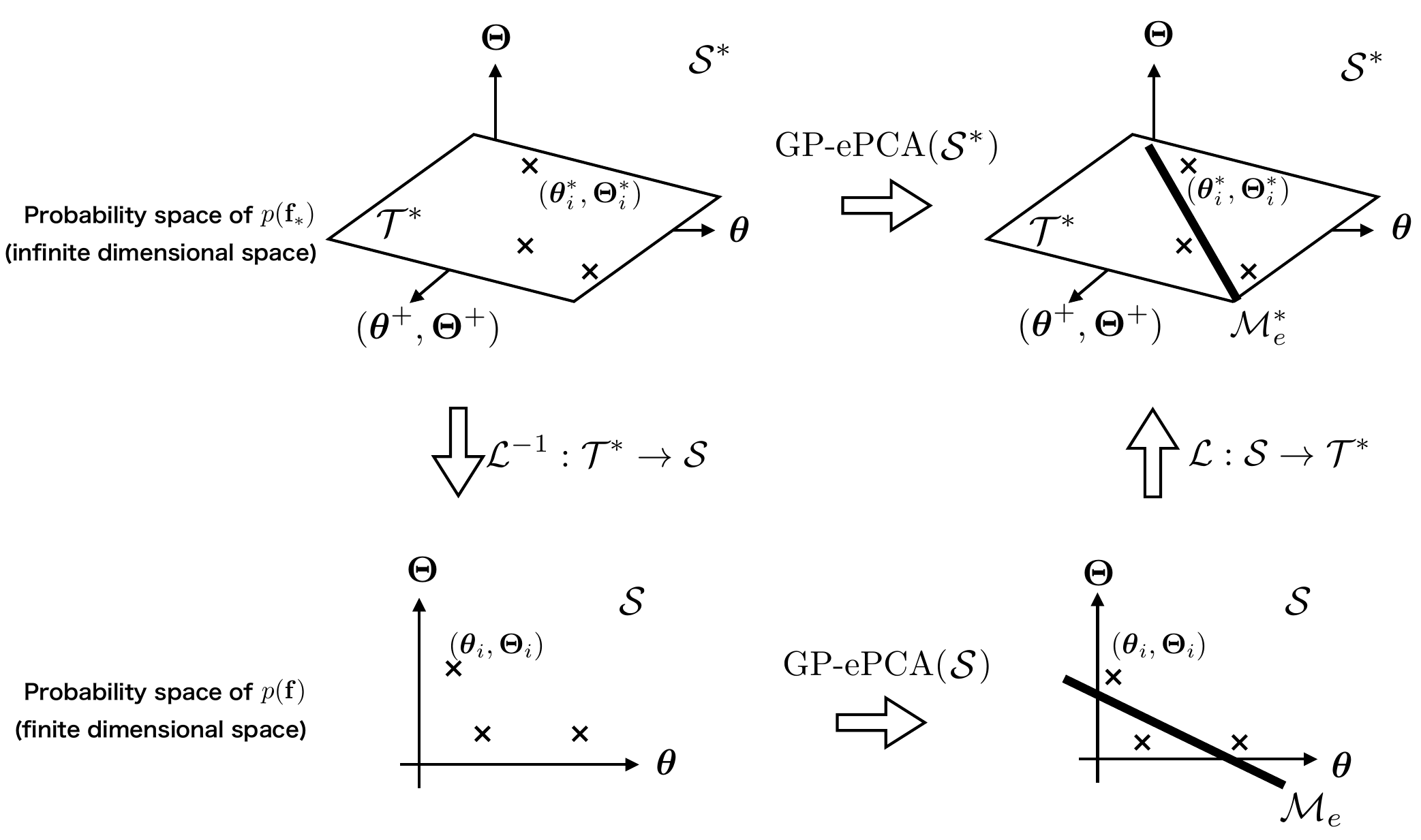} \\
    \caption{Concept of GP-ePCA. An e-flat subspace $\cM^\ast_e \subset \cS^\ast$ is shown to be identical to an e-flat subspace $\cM_e \subset \cS$ to $\cT^\ast$ through linear map $\cL$.}
    \label{concept_of_GP-PCA}
\end{figure}

Similar to e-PCA and m-PCA, we consider two types of GP-PCA: GP-ePCA and GP-mPCA. In this study, we only explain the GP-ePCA, but the same argument holds for GP-mPCA. Let $p(f \mid \vy_i)$ be a GP posterior obtained from $\{X_i,\vy_i\}$. When a set of posteriors $\cP = \{p(f \mid \vy_i)\}^I_{i=1}$ is given, the task of GP-ePCA is to estimate an e-flat subspace minimizing KL divergence between GP posteriors and their corresponding points on the subspace. However, it is nontrivial to define a structure of GPs since GP has an infinite-dimensional parameter.

This study shows that a set of GP posteriors is a finite-dimensional dually flat space under the assumption that each posterior has the same prior and reduces the task of GP-ePCA to a task of estimating a subspace on finite space. To explain our approach, we introduce the two probabilistic spaces shown in Fig~\ref{concept_of_GP-PCA}. One is a space consisting of Gaussian distributions for an output vector $\vf:=f(X)$ corresponding to a training set $X=\bigcup^I_{i=1}X_i$. The other is a space consisting of Gaussian distributions for an output vector $\vf_\ast:=f(X_\ast)$ corresponding to a set $X_\ast$ which is a union of the training input set and an arbitrary test input set $X_+$. The former is denoted by $\cS$ and the latter is denoted by $\cS^\ast$. Both $\cS$ and $\cS^\ast$ are dually flat spaces since they are a set of Gaussian distributions. Note that $\cS^\ast$ can be regarded as an infinite-dimensional space since the cardinal of $X_+$ can be any number. In our approach, we define a space consisting of GP posteriors as a subspace on $\cS^\ast$ denoted by $\cT^\ast$. Then, we estimate an e-flat subspace $\cM_e$ on $\cS$ and transform $\cM_e$ to $\cS^\ast$ using an affine map $\cL : \cS \rightarrow \cT^\ast$ instead of estimating an e-flat subspace $\cM^\ast_e$ on $\cS^\ast$. Since there is no guarantee that $\cL(\cM_e)$ is equivalent to $\cM^\ast_e$, this study proves this.

In the present section, after defining $\cT^\ast$ and GP-ePCA, we prove that $\cM^\ast_e$ and $\cL(\cM_e)$ are equivalent. Next, we describe the standard algorithm and its sparse approximation algorithm.

\subsection{Definition of the structure of GP posteriors and GP-ePCA}

Let $X$ and $X_+$ be a union set of $\{X_i\}^I_{i=1}$ and test set. We consider estimating $p(\vf \mid \vy_i)$ given $\{X_i,\vy_i\}$ in each task. Then, $i$-th task's predictive distribution for $X_+$ is derived as $q(\vf_+\mid\vy_i)=\int p(\vf_+\mid\vf)p(\vf \mid \vy_i)d\vf$. Suppose that GP posteriors have a common prior, then $q(\vf_+\mid\vy_i)$ is determined uniquely given $p(\vf \mid \vy_i)$. From this fact, the affine subspace spanned by GP posteriors is defined as follows:
\begin{definition}
Let $X$, $X_+$ and $X_\ast$ be an input set, test set, and a union set of input and test sets, respectively. We denote the size of $X$ by $N$. Let $p(\vf\mid\vrho)$ be a Gaussian distribution with $\vrho$, where $\vrho$ is a pair of $N$-dimensional vector $\vmu$ and $N \times N$ positive-definite symmetric matrix $\vSigma$. Then, a probability space consisting of GP posteriors corresponding to $f(X_\ast)$ with a common prior is defined by the following equation:
\begin{equation}
    \cT^\ast=\Set{q(\vf_+,\vf\mid\vrho)|q(\vf_+,\vf\mid\vrho) = p(\vf_+\mid\vf)p(\vf\mid\vrho), \forall \vrho},
    \label{eq_GP_posterior_space}
\end{equation}
where $p(\vf_+\mid\vf)$ is a conditional distribution of the prior and $p(\vf \mid \vrho)$ is any Gaussian distribution with a parameter $\vrho$. In particular, when $X=X_\ast$ holds (i.e., when $X_+$ is an empty set), $\cS^\ast$ and $\cT^\ast$ are denoted by $\cS$ and $\cT$, respectively.
\label{model_manifold_of_GP}
\end{definition}

Satisfying the assumption, $p(\vf_+,\vf\mid \vy_i)$ is contained in $\cT^\ast$. Let $\vrho_i=\{\vmu_i,\vSigma_i\}$ be a parameter of $p(\vf_+,\vf\mid \vy_i)$. $\vrho_i$ can be described as follows:
\begin{align}
    \vmu_i &= \vmu_0 + \vK_i(\vK_{ii}+\beta^{-1}\vI)\inv(\vy_i-\vmu_{0i}), \label{eq_rho_mean} \\
    \vSigma_i &= \vK - \vK_i(\vK_{ii}+\beta^{-1}\vI)\inv \vK\T_i, \label{eq_rho_cov}
\end{align}
where $\vK=k(X,X)$, $\vK_i=k(X,X_i)$, $\vK_{ii}=k(X_i,X_i)$, $\vmu_0=\mu_0(X)$ and $\vmu_{0i}=\mu_0(X_i)$. Therefore, we can define a space of GP posteriors as $\cT^\ast$.

Since $\cS^\ast$ is a dually flat space, $p(\vf_\ast) \in \cS^\ast$ can be represented by e-coordinate and m-coordinate denoted by $\vxi^\ast$ and $\vzeta^\ast$, respectively. We denote e-coordinate and m-coordinate for a point on $\cT^\ast$ parameterized by $\vrho$ as $\vxi^\ast(\vrho)$ and $\vzeta^\ast(\vrho)$, respectively. From the definition of $\cT^\ast$, when $X_\ast=X$, $\cS=\cT$ holds since $\vmu_\ast=\vmu$ and $\vSigma_\ast=\vSigma$ hold. This means that $\cT(=\cS)$ is also a dually flat space. Therefore, we denote e-coordinate and m-coordinate of $p(\vf \mid \vrho) \in \cT$ by $\vxi=\vxi(\vrho)$ and $\vzeta=\vzeta(\vrho)$, respectively. 

By using the definition of $\cT^\ast$, we define GP-ePCA in the respective spaces of $\cS^\ast$ and $\cS$.
\begin{definition}
Let $\{\vxi^\ast(\vrho_i)\}^I_{i=1}$ be a set of GPs on the $\cT^\ast$. Then, the objective function of GP-ePCA on $\cS^\ast$ is defined as follows:
\begin{align}
    \hat{\vW}^\ast,\hat{\vU}^\ast &= \argmin_{\vW^\ast,\vU^\ast}E^\ast(\vW^\ast,\vU^\ast) \nonumber \\
    &= \argmin_{\vW^\ast,\vU^\ast}\sum^I_{i=1}\KL{\vxi^\ast(\vrho_i)}{\vxi^\ast(\vw^\ast_i,\vU^\ast)}.
    \label{eq_obj_func_GP_PCA_ast}
\end{align}
GP-ePCA estimating e-flat submanifold $\cM^\ast_e\subset \cS^\ast$ minimizing Eq.~\eqref{eq_obj_func_GP_PCA_ast} is called GP-ePCA$(\cS^\ast)$. Here, $\vxi^\ast(\vw^\ast_i,\vU^\ast)$ is e-coordinate of $\cM^\ast_e$ denoted by a linear combination of $\mathbf{U}^\ast=(\vu^\ast_0,\vu^\ast_1,\vu^\ast_2,\ldots,\vu^\ast_L)\T$ with weight $(1,{\vw^\ast_i}\T)$, where $\vw^\ast_i:=(w^\ast_1,w^\ast_2,\ldots,w^\ast_L)\T$.

Similarly, when $\{\vxi(\vrho_i)\}^I_{i=1}$ is observed, we call the ePCA minimizing the following equation GP-ePCA$(\cS)$:
\begin{align}
    \hat{\vW},\hat{\vU} &= \argmin_{\vW,\vU}E(\vW,\vU) \nonumber \\
    &= \argmin_{\vW,\vU}\sum^I_{i=1}\KL{\vxi_i}{\vxi(\vw_i,\vU)}.
    \label{eq_obj_func_GP_PCA}
\end{align}
Here, $\vxi_i$ and $\vxi(\vw_i,\vU)$ are e-coordinate of $\cS$ and $\cM_e$, which is a linear combination of $\mathbf{U}=(\vu_0,\vu_1,\vu_2,\ldots,\vu_L)\T$ with weight $(1,\vw\T_i)$, where $\vw_i:=(w_1,w_2,\ldots,w_L)\T$.
\label{definition_GP-PCA}
\end{definition}

In this study, we guarantee that GP-ePCA($\cS^\ast$) is equivalent to GP-ePCA($\cS$) by the following theorem.

\begin{theorem}
Let $\cM^\ast_e$ and $\cM_e$ be an e-flat subspace on $\cS^\ast$ minimizing Eq.~\eqref{eq_obj_func_GP_PCA_ast} and an e-flat subspace on $\cS$ minimizing Eq.~\eqref{eq_obj_func_GP_PCA}, respectively. Then, there is an affine map $\mathcal{L} : \cS \rightarrow \cT^\ast$ satisfying the following equation:
\begin{equation}
    \cM^\ast_e = \mathcal{L}(\cM_e).
    \label{eq_main_theorem}
\end{equation}
\label{main_theorem}
\end{theorem}
We prove the theorem below.

\subsection{Proof of Theorem~\ref{main_theorem}}
The proof of the Theorem~\ref{main_theorem} is composed of the proof of the following three statements:
\begin{enumerate}
    \renewcommand{\theenumi}{S\arabic{enumi}}
    \renewcommand{\labelenumi}{(\theenumi)}
    \item For $\forall \vrho$, there is $\cL$ satisfying $\vxi^\ast(\vrho) = \cL(\vxi(\vrho))$. \label{Statement1}
    \item For $\forall \vrho$, $\vrho'$, $\KL{\vxi^\ast(\vrho)}{\vxi^\ast(\vrho')}=\KL{\vxi(\vrho)}{\vxi(\vrho')}$ holds. \label{Statement2}
    \item For a subspace $\cM^\ast_e \subset \cS^\ast$ minimizing $E^\ast(\vW^\ast,\vU^\ast)$, $\cM^\ast_e\subset\cT^\ast$ holds. \label{Statement3}
\end{enumerate}

From (\ref{Statement1}) and (\ref{Statement2}), denoting a subspace minimizing Eq.~\eqref{eq_obj_func_GP_PCA} by $\cM_e$, we can prove that $\cL(\cM_e)$ also minimizesf Eq.~\eqref{eq_obj_func_GP_PCA_ast} in a set of subspaces on $\cT^\ast$. However, since a subspace minimizing Eq.~\eqref{eq_obj_func_GP_PCA_ast} does not always lie on $\cT^\ast$, we confirm this by (\ref{Statement3}).

To prove the statements, we present the following lemmas:
\begin{lemma}
Let $\vrho$ be a parameter of $\cT^\ast$. Then, there is an affine map $\mathcal{L}:\cS \rightarrow \cT^\ast$ satisfying the following equation:
\begin{equation}
    \vxi_\ast(\vrho) = \mathcal{L}(\vxi(\vrho))
\end{equation}
\label{lemma_e_flatness}
\end{lemma}
\begin{proof}
The proof is shown in Appendix~\ref{sec_appendix_proof_lemmas}
\end{proof}

\begin{lemma}
Let $\vrho$ and $\vrho'$ are two arbitrary parameters, and let us take two points $q(\vf_\ast\mid \vrho)$ and $q(\vf_\ast\mid \vrho')$ in $\cT^\ast$, and $q(\vf\mid \vrho)$ and $q(\vf\mid \vrho')$ in $\cT$. Then, the following equation holds:
\begin{equation}
    \KL{q(\vf_\ast\mid \vrho)}{q(\vf_\ast\mid \vrho')}=\KL{q(\vf\mid \vrho)}{q(\vf\mid \vrho')}
\end{equation}
\label{lemma_equivalent_of_kl_divergence}
\end{lemma}
\begin{proof}
The proof is shown in Appendix~\ref{sec_appendix_proof_lemmas}
\end{proof}

\begin{lemma}
Suppose $\cS^\ast$ be a dually flat manifold and $\cT^\ast \subset \cS^\ast$ be a $K$-dimensional submanifold.
If $\cT^\ast$ is a dually flat and a set of points $P = \{p(\vf^\ast\mid \vrho_1),\ldots,p(\vf^\ast\mid \vrho_I)\}\in \cT^\ast$,
the $L$-dimensional e-flat submanifold $\cM^\ast_e$ minimizing Eq.~\eqref{eq_obj_func_GP_PCA_ast} for $P$ is included in $\cT^\ast$ when $L \le K$. 
\label{lemma_equivalent_gppca}
\end{lemma}
\begin{proof}
The proof is shown by Appendix~\ref{sec_appendix_proof_lemmas}
\end{proof}

\begin{lemma}
Let $\vrho$ be a parameter of $\cT^\ast$. Then, there is a linear mapping $\mathcal{L}:\cT \rightarrow \cT^\ast$ satisfying the following equation:
\begin{equation}
    \vzeta_\ast(\vrho) = \mathcal{L}(\vzeta(\vrho))
\end{equation}
\label{lemma_m_flatness}
\end{lemma}
\begin{proof}
The proof is shown by Appendix~\ref{sec_appendix_proof_lemmas}
\end{proof}

The proofs of (\ref{Statement1}) and (\ref{Statement2}) are obvious from Lemma~\ref{lemma_e_flatness} and Lemma~\ref{lemma_equivalent_of_kl_divergence}. From Lemma~\ref{lemma_equivalent_gppca}, (\ref{Statement3}) can be proved by showing that $\cT^\ast$ is a dually flat for arbitrary test set $X_+$. When $X=X_\ast$, (i.e., the test set is empty), then $\cT$ is a dually flat since $\cT=\cS$. When $X\subset X_\ast$, by the linear relation proved in Lemma~\ref{lemma_e_flatness} and Lemma~\ref{lemma_m_flatness}, Lemma~\ref{lemma_equivalent_gppca} also holds in the general case. Thus, Theorem~\ref{main_theorem} is proved.

\section{Algorithm of GP-ePCA}
From the above discussion, GP-ePCA$(\cS^\ast)$ can be reduced to GP-ePCA$(\cS)$. In this Section, we explain a concrete algorithm of GP-ePCA$(\cS)$, its sparse approximation method and hyperparameter optimization.
\subsection{Exact GP-ePCA}
Let $X_i \in \mathcal{X}^{N_i}$ and $\vy_i \in \mathbb{R}^{N_i}$ be a training input and corresponding output dataset of $i$-th task, respectively, where $N_i$ is the size of $X_i$. We denote a union set of the input sets by $X$, i.e., $X=\bigcup^I_{i=1}X_i$ and define the probability space of GP posteriors as Eq.~\eqref{eq_GP_posterior_space}. Then, we denote the GP posterior given $(X_i,\vy_i)$ by $q(f\mid \vrho_i)$. From Theorem~\ref{main_theorem}, the task of GP-ePCA$(\cS)$ is to estimate a subspace $\cM_e$ for $\{p(\vf\mid \vrho_i)\}^I_{i=1}$ and transform $\cM_e$ to $\cT^\ast$. 

In training phase, GP-ePCA calculates the $\{\vrho_i\}^I_{i=1}$ and transforms the m-coordinates $\vZeta := (\vzeta_1,\vzeta_2,\cdots,\vzeta_I)$. The $\vrho_i$ is calculated using Eqs.~\eqref{eq_rho_mean} and \eqref{eq_rho_cov} and from $\vzeta_i=(\veta_i,{\rm vec}(\vEta_i))$ is transformed from $\vrho_i$ by $\veta_i=\vmu_i$ and $\vEta_i=\vSigma_i+\vmu_i\vmu_i\T$. Next, GP-ePCA$(\cS)$ estimates the subspace $\cM_e$ using e-PCA. That is, estimating ${\hat \vW}$ and ${\hat \vU}$ minimizing Eq.~\eqref{eq_obj_func_GP_PCA} through gradient descent iterations. Algorithm~\ref{alg_standard_GP_ePCA} shows the summary of the algorithm. 

In the prediction phase, GP-ePCA predict outputs corresponding to a test data $x$ using the following equations:
\begin{align*}
    \mu_i(x) =& \vmu_0(x)+\vk\T(x)\vK\inv\left(\hat{\vTheta}\inv_i\hat{\vtheta}_i-\vmu_0\right) \\
    \sigma_i(x,x') =& k(x,x')+\vk\T(x)\vK\inv\left(-\frac{1}{2}\hat{\vTheta}\inv_i-\vK\right)\vK\inv\vk(x) 
\end{align*}

Since this algorithm requires calculating the inverse matrix in each task, the calculation cost of the algorithm becomes $\mathcal{O}(IN^3)$, where $N:=\sum^I_{i=1}N_i$. Since this algorithm is impractical, we derive a faster approximation below.

\begin{algorithm}[H]
\caption{Algorithm of Exact GP-ePCA}
\label{alg_standard_GP_ePCA}
\begin{algorithmic}
    \State Given $\{(X_i,\vy_i)\}^I_{i=1}$, kernel $k$ and $\beta$.
    \State Initialize $\vU$ and $\vW$
    \For{$i = 1 \, \ldots \, I$}
        \State $\vmu_i \leftarrow \vmu_0 + \vK_i(\vK_{ii}+\beta^{-1}\vI)\inv(\vy_i+\vmu_{i0})$ 
        \State $\vSigma_i \leftarrow \vK - \vK_i(\vK_{ii}+\beta^{-1}\vI)\inv \vK\T_i$
        \State $\vzeta_i \leftarrow (\vmu_i,{\rm vec}(\vSigma_i+\vmu_i\vmu_i\T))$
    \EndFor
    \While{stopping criterion is met}
        \State $\vW^{\rm (new)} \leftarrow \vW^{\rm (old)}-\varepsilon(\hat{\vZeta}-\vZeta)\T\vU$
        \State $\vU^{\rm (new)} \leftarrow \vU^{\rm (old)}-\varepsilon\vW(\hat{\vZeta}-\vZeta)$
        \State $\vU^{\rm (new)}$ is orthonormalized by QR decomposition
        \For{$i = 1 \, \ldots \, I$}
            \State $\hat{\vxi}_i \leftarrow \hat{\vw}\T_i\hat{\vU}$
        \EndFor
    \EndWhile
\end{algorithmic}
\end{algorithm}

\subsection{Sparse GP-ePCA}
Most sparse approximation methods for GP reduce a calculation cost by approximating the gram matrix for input set using inducing points~\citep{Liu2020}. Let $X_m$ and $\vK$ be a set of inducing points and gram matrix between inputs, respectively. The gram matrix is approximated as
\begin{equation*}
    \vK \approx \vK_m\vK\inv_{mm}\vK\T_m,
\end{equation*}
where $\vK_m=k(X,X_m)$, $\vK_{mm}=k(X_m,X_m)$. By using this approximation, we consider a set of GPs for $\vf_m:=f(X_m)$ instead of a set of GPs for $f(X)$. Denoting the set of GPs for $\vf_m$ by $\cS_m$, the sparse GP-ePCA estimates a subspace on $\cS_m$ and transforms the subspace to $\cT^\ast$. Then, we reduce the calculation cost of GP-ePCA from $\mathcal{O}(IN^3)$ to $\mathcal{O}(Im^3)$, where $m$ is the size of inducing points. 

We adopt a sparse GP based on variational inference proposed by Titsias~\citep{Titsias2009}. The variational inference-based sparse GP minimizes the KL-divergence between a true posterior $p(\vf,\vf_m \mid \vy)$ and variational distribution $q(\vf, \vf_m)$, that is,
\begin{equation*}
    \KL{q(\vf, \vf_m)}{p(\vf,\vf_m \mid \vy)}=\int q(\vf, \vf_m) \ln \frac{q(\vf, \vf_m)}{p(\vf,\vf_m \mid \vy)}d\vf d\vf_m.
\end{equation*}
Then, the variational distribution minimizing the equation is derived as follows:
\begin{align*}
    q(\vf_m) =& \mathcal{N}(\vf_m \mid \vmu,\vSigma) \\
    \vmu =& \vmu_{m0} + \vK_{mm}\vA\inv_{mm}\vK\T_{m}(\vy-\vmu_0) \\ 
    \vSigma =& \beta\inv\vK_{mm}\vA\inv_{mm}\vK_{mm},
\end{align*}
where $\vA_{mm}=\beta\inv\vK_{mm}+\vK_m\vK\T_m$. The predictive distribution for new input $x_+$ is as follows:
\begin{align*}
    q(f(x_+)) =& \mathcal{N}(f(x_+) \mid \mu(x_+),\sigma(x_+,x_+)) \\
    \mu(x_+) =& \mu_{0}(x_+) + \vk\T_m(x_+)\vK\inv_{mm}\vmu \\ 
    \sigma(x_+,x_+) =& \beta\inv\vk\T_m(x_+)\vK\inv_{mm}\vSigma\vK\inv_{mm}\vk_m(x_+),
\end{align*}

We regard $\vmu$ and $\vSigma$ as a parameter of Eq.~\eqref{eq_GP_posterior_space}. That is, denoting a parameter of $i$-th task's variational distribution by $\vrho_i=\{\vmu_i,\vSigma_i\}$, the sparse GP-ePCA estimates a subspace minimizing Eq.~\eqref{eq_obj_func_GP_PCA} for $\{\vrho_i\}^I_{i=1}$ and transforms the subspace to $\cT^\ast$ by the affine map $\cL$. 

In practice, to stabilize the sparse GP-ePCA, we re-parametrize $\vrho_i$ as follows:
\begin{align*}
    \vmu' =& \vK\inv_{mm}\vmu, \\
    \vSigma' =& \vK\inv_{mm}\vSigma\vK\inv_{mm}.
\end{align*}
We denote a space of $\vrho'=\{\vmu',\vSigma'\}$ by $\cS'_m$.
Letting $\vtheta'={\vSigma'}\inv\vmu'$, $\vTheta' = -\frac{1}{2}{\vSigma'}\inv$, $\vtheta=\vSigma\inv\vmu$ and $\vTheta = -\frac{1}{2}\vSigma\inv$, the following relationships between $\vxi'(\vrho)=(\vtheta',{\rm vec}(\vTheta'))$ and $\vxi(\vrho)=(\vtheta,{\rm vec}(\vTheta))$ hold.
\begin{align*}
    \vTheta =& \vK\inv_{mm}\vTheta'\vK\inv_{mm}. \\
    \vtheta =& \vK\inv_{mm}\vTheta'\vK\inv_{mm}\vmu_0+\vK\inv_{mm}\vtheta'.
\end{align*}
Furthermore, using the above equations, we can show the equivalence between the KL-divergence of $\vxi'$ and that of $\vxi$. That is, for any $\vrho_i$ and $\vrho_j$, the following equation holds:
\begin{equation*}
    \KL{\vxi'(\vrho_i)}{\vxi'(\vrho_j)}=\KL{\vxi(\vrho_i)}{\vxi(\vrho_j)}
\end{equation*}
From the above relationships, $\cS_m$ and $\cS'_m$ are isomorphic. Therefore, we estimate a subspace on $\cS'_m$ instead of estimating a subspace on $\cS_m$. The algorithm is summarized by algorithm~\ref{alg_sparse_approx_GP_ePCA}.

\begin{algorithm}[H]
\caption{Algorithm of Sparse GP-ePCA}
\label{alg_sparse_approx_GP_ePCA}
\begin{algorithmic}
    \State Given $\{(X_i,\vy_i)\}^I_{i=1}$, kernel $k$, $\beta$ and $X_m$.
    \State Initialize $\vU$ and $\vW$
    \For{$i = 1 \, \ldots \, I$}
        \State $\vA_{mm} \leftarrow \beta\inv\vK_{mm}+\vK_m\vK\T_m$
        \State $\vmu_i \leftarrow \vA\inv_{mm}\vK\T_{m}(\vy-\vmu_0)$
        \State $\vSigma_i \leftarrow \beta\inv\vA\inv_{mm}$
        \State $\vzeta_i \leftarrow (\vmu_i,{\rm vec}(\vSigma_i+\vmu_i\vmu_i\T))$
    \EndFor
    \While{stopping criterion is met}
        \State $\vW^{\rm (new)} \leftarrow \vW^{\rm (old)}-\varepsilon(\hat{\vZeta}-\vZeta)\T\vU$
        \State $\vU^{\rm (new)} \leftarrow \vU^{\rm (old)}-\varepsilon\vW(\hat{\vZeta}-\vZeta)$
        \State $\vU^{\rm (new)}$ is orthonormalized by QR decomposition
        \For{$i = 1 \, \ldots \, I$}
            \State $\hat{\vxi}_i \leftarrow \hat{\vw}\T_i\hat{\vU}$
        \EndFor
    \EndWhile
\end{algorithmic}
\end{algorithm}

\subsection{Hyperparameter optimization}

In this study, we use a hierarchical Bayes-based GP (HBGP) proposed by \citet{Yu2005} for the hyperparameters optimization method under the assumption the GP posteriors lie on a subspace estimated by GP-ePCA. In HBGP, $\vmu_0$ and $\vK$ are estimated by maximizing the following likelihood function:
\begin{equation}
    \prod^I_{i=1}p(\vy_i | \vmu_0,\vK)=\prod^I_{i=1}p(\vy_i|\vf_i)p(\vf_i | \vmu_0,\vK)p(\vmu_0,\vK),
    \label{eq_hierarchical_bayes_likelihood}
\end{equation}
where $\vf_i:=f_i(X)$, $p(\vf_i | \vmu,\vSigma)=\mathcal{N}(\vf_i | \vmu_0,\vK)$ and $p(\vmu_0,\vK)$ is the following normal-inverse-Wishart distribution:
\begin{equation*}
    p(\vmu_0,\vK) = \mathcal{N}(\vmu_0 | \vzero, \frac{1}{\pi}\vK)\mathcal{IW}(\vK | \tau, \vK_0).
    \label{eq_transductive_hyper_prior}
\end{equation*}
Here, $k_0$ is a kernel function and $\vK_0=k_0(X,X)$. $\vmu$ and $\vK$ can be estimated by using EM algorithm. However, we cannot calculate a value of the prior corresponding to a new input data. Therefore, we assume that $\vf$ is represented by linear combination of $\vK_0$, that is, \begin{equation*}
    \vf = \vK_0\va.
\end{equation*}
Assuming that $\vf$ is generated by Eq.~\eqref{eq_hierarchical_bayes_likelihood}, let a prior of $\va$ be $p(\va | \vmu_a,\vK_a)$, then the following relationships hold~\citep{Yu2005}.
\begin{align*}
    p(\va | \vmu_a,\vK_a) =& \mathcal{N}(\vmu_a | \vzero, \frac{1}{\pi}\vK_a)\mathcal{IW}(\vK_a | \tau, \vK\inv_0) \\
    \vmu_0 =& \vK_0 \vmu_a \\
    \vK =& \vK_0 \vK_a \vK_0. \\
\end{align*}
By using the relationships, instead of maximizing Eq.~\eqref{eq_hierarchical_bayes_likelihood}, we estimate $\vmu_a$ and $\vK_a$ by maximizing the following equation:
\begin{equation*}
    \prod^I_{i=1}p(\vy_i | \vmu_a,\vK_a)=\prod^I_{i=1}p(\vy_i|\va_i)p(\va_i | \vmu_a,\vK_a)p(\vmu_a,\vK_a),
\end{equation*}
The likelihood is maximized by the following EM algorithm.
\mbox{}\\
\noindent{\textbf{E-Step}}\mbox{}\\
\begin{align*}
    \vxi(\vw_i,\vU) =& (\vtheta_i,{\rm vec}(\vTheta_i)) = (1,\vw)\T_i\vU \\
    \vmu'_i =& -2\vTheta_i \vtheta_i \\
    \vSigma'_i =& -\frac{1}{2}\vTheta\inv_i
\end{align*}
\mbox{}\\
\noindent{\textbf{M-Step}}\mbox{}\\
\begin{align*}
    \vmu_a &= \frac{1}{\pi+I}\sum^I_{i=1}\vmu'_i \\
    \vK_a &= \frac{1}{\pi+I}\sum^I_{i=1}\left\{\pi\vmu_a\vmu\T_a+\tau\vK\inv_0+\sum^{I}_{i=1}\vSigma'_i+\sum^I_{i=1}(\vmu'_i-\vmu_a)(\vmu'_i-\vmu_a)\T\right\} \\
    \beta\inv &= \frac{1}{N}\sum^I_{i=1}\|\vy_i-k_0(X_i,X)\vmu'_i\|^2+{\rm Tr}[k_0(X_i,X)\vK_ak_0(X,X_i)]
\end{align*}
Here, we modify the E-step from HBGP since $\vxi(\vw_i,\vU)$ is more accurate than $\vxi_i$. The hyperparameter is updated after GP-ePCA is converged.

\subsection{Related works}

Dimensionality reduction techniques for probability distributions have been proposed in various fields. For example, there are dimension reduction techniques of a set of categorical distributions~\citep{Hofmann2001} and a set of mixture models~\citep{Akaho2008,Covarrubias2013}. Especially, e-PCA and m-PCA are closely related to the present study~\citep{Collins2001,Akaho2004}. e-PCA and m-PCA are proposed in the context of information geometry for the dimension reduction method of a set of exponential distribution families, which becomes the basic framework for conducting this study. This study differs from previous studies in that it deals with GP sets that are infinite-dimensional stochastic processes.

This study interprets meta-learning for GP from the information geometry viewpoint. Transfer learning and meta-learning are often addressed from the information geometry perspective~\citep{Takano2016,Waytowich2016,Flennerhag2019}. However, to our best knowledge, there is no research of meta-learning for GP addressed from the information geometry viewpoint.

GP-PCA can also be interpreted as a functional PCA (fPCA). fPCA is a method for estimating eigenfunctions from a set of functions~\citep{Shang2014}. Let $\{f_i\}^I_{i=1}$ be a set of functions. fPCA estimates eigenfunctions to minimize the following objective function:
\begin{align*}
    F_{\rm fPCA}=&\sum^I_{i=1}\int (f_i(x)-h(x)-\bar{f}(x))^2p(x)dx, \\
    & s.t. \quad \int h^2(x)p(x)dx=1,
\end{align*}
where $\bar{f}=\frac{1}{I}\sum^I_{i=1}f_i(x)$. In fPCA, each function is represented as a linear combination of $M$ basis functions. Let $\vg(x)=(g_1(x),g_2(x),\ldots,g_I(x))\T$, $f_i$ is obtained as
\begin{align*}
    f_i(x) =& \sum^M_{m=1}r_{im}g_m(x) \\
    =& \vr\T_i\vg(x),
\end{align*}
where $\vr_i=(r_{i1},r_{i2},\ldots,r_{iM})\T$. $h(x)$ is represented as a linear combination of $\vg(x)$: $h(x)=\vs\T\vg(x)$. By using the equations, the objective function of fPCA is rewritten as follows:
\begin{align*}
    F_{\rm fPCA}=&\sum^I_{i=1} (\vr_i-\vs-\bar{\vr})\T\vG(\vr_i-\vs-\bar{\vr}), \\
    & s.t. \quad \vs\T\vG\vs=1,
\end{align*}
where $\bar{\vr}=1/I\sum^I_{i=1}\vr_i$ and $\vG=\int\vg(x)\vg\T(x)p(x)dx$. In practice, $f_i$ is estimated using linear regression from the dataset. That is, the fPCA algorithm consists of two processes: estimating $f_i$ from $\{X_i,\vy_i\}$ and estimating $h$ from $\{f_i\}^I_{i=1}$. When $f_i$ is estimated as GP, GP-PCA is equivalent to fPCA. Therefore, GP-PCA can be interpreted as fPCA considering the estimated function and confidence of the function.

\section{Experimental results}
In this Section, we demonstrate the effectiveness of the proposed method. We compare the performance of sparse GP-ePCA, ICM~\citep{Bonilla2007}, Hierarchical Bayes-based GP (HBGP)~\citep{Yu2005} and Single-task sparse GP by using an artificial dataset and two real datasets, where ICM is implemented by GPy\footnote{\normalsize https://github.com/SheffieldML/GPy}. We also compare the calculation time of exact GP-ePCA and sparse GP-ePCA.
\subsection{Evaluation of performance by using an artificial dataset}
In this experiment, we compare the performance of the proposed method to the other methods in training task and test task. The artificial dataset is generated from the following equations:
\begin{align}
    y_{in} &= z_i \sin(4\pi x_{in}) + 3(1-z_i)(-(x_{in}-1)^2+1) +\varepsilon_{in} \label{eq_output_generator} \\
    x_{in} &\sim U(0,1) \label{eq_input_generator}
\end{align}
where $U[a,b]$ means a continuous uniform distribution of interval $[a,b]$, $\vz_i$ is a latent variable of $i$-th task, and $\varepsilon_{in}$ is a Gaussian noise with mean $0$ and variance $\beta\inv=0.2^2$. 

To verify the performance of GP-ePCA for training and test tasks, training data and test data of training and test tasks are necessary. In training tasks, after $z_i$ is sampled for $50$ points according to Eq.~\eqref{eq_input_generator}, we sample $N$ training data and $100$ test data of each task according to Eq.~\eqref{eq_output_generator}. Figure~\ref{fig_fitting}(a) and (b) show the training sample and true functions of training tasks, respectively. In test tasks, after sampling latent variables of the test tasks at $100$ samples, we sample $N$ training data and $100$ test data for test tasks. Then, the training data of the test tasks are used to determine a point on an estimated subspace in GP-ePCA. In GP and HBGP, the posteriors for the test tasks are estimated from the training data of the task. In ICM, the test tasks are regarded as have been obtained together with the training tasks and are learned together with the training tasks. The performance of each method is evaluated using an average of root mean square error (RMSE) for test data in each task. We calculate the average and standard deviation when $T=5$ times iterates.

GP has hyperparameters, which are kernel and variance of observation noise. In this experiment, we use RBF kernel $k(x,x')=\exp(-(x-x')^2/2l^2)$. Then, the hyperparameters of GP are a length scale $l$ of the RBF kernel and variance of an observation noise $\beta\inv$. These hyperparameters are optimized by maximizing the marginal likelihood in GP and ICM. In GP-ePCA, a rank of subspace is also a hyperparameter. In these experiments, we set $L=1$ since the functions of the artificial data are varied by a one-dimensional parameter. The effectiveness of the rank is validated by the next experiments.

Figure~\ref{fig_fitting} shows representative results of each method. In this case, each task has five samples: single-task GP cannot predict output in an area without training data since a mean function of GP is close to prior (Fig.~\ref{fig_fitting}(c)). In ICM, while most of tasks are able to estimate the predictive function, the prediction function cannot be estimated  for some tasks, shown by Fig.~\ref{fig_fitting}(d). On the other hand, HBGP and GP-ePCA can predict output in an area without training data since mean functions are smoothed by transferring knowledge from other tasks (Fig.~\ref{fig_fitting}(c)). This result implies that estimating the mean function is important in few-shot learning. Figure~\ref{fig_gene_error_artificial} shows the RMSE of the methods for training and test tasks when $N$ is varied. These results shows that GP-ePCA performs better than the other methods in both training and test tasks. Especially, even in case that training sample size of each task is small, the RMSE of the proposed method remains low compared to other methods in both of training tasks and test tasks. 

\begin{figure}
    \centering
    \begin{tabular}{ccc}
    \includegraphics[width=4.5cm]{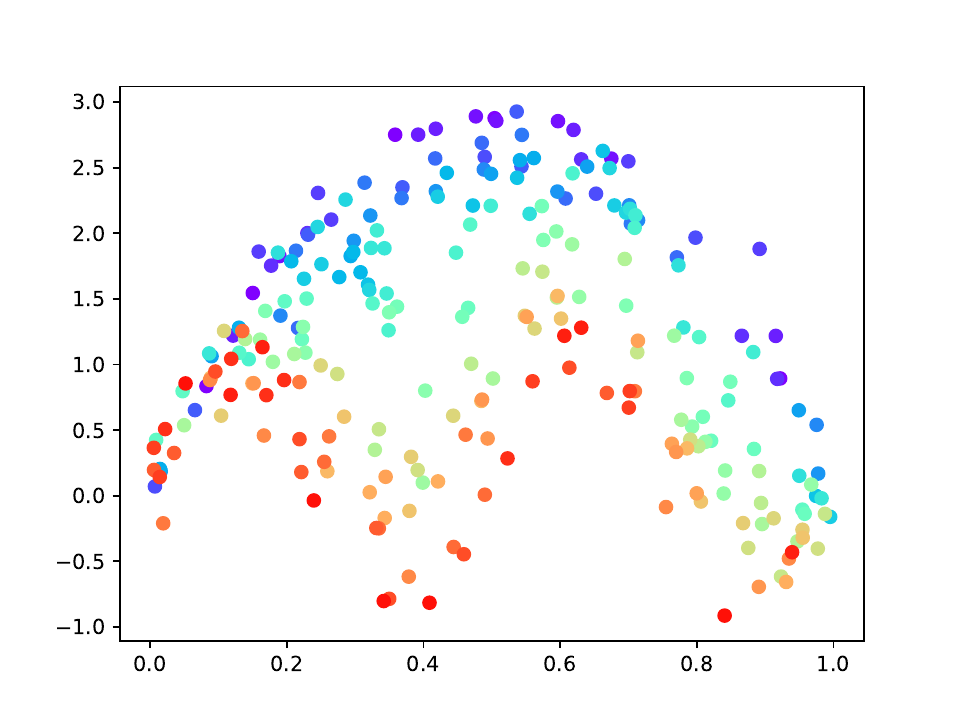}&
    \includegraphics[width=4.5cm]{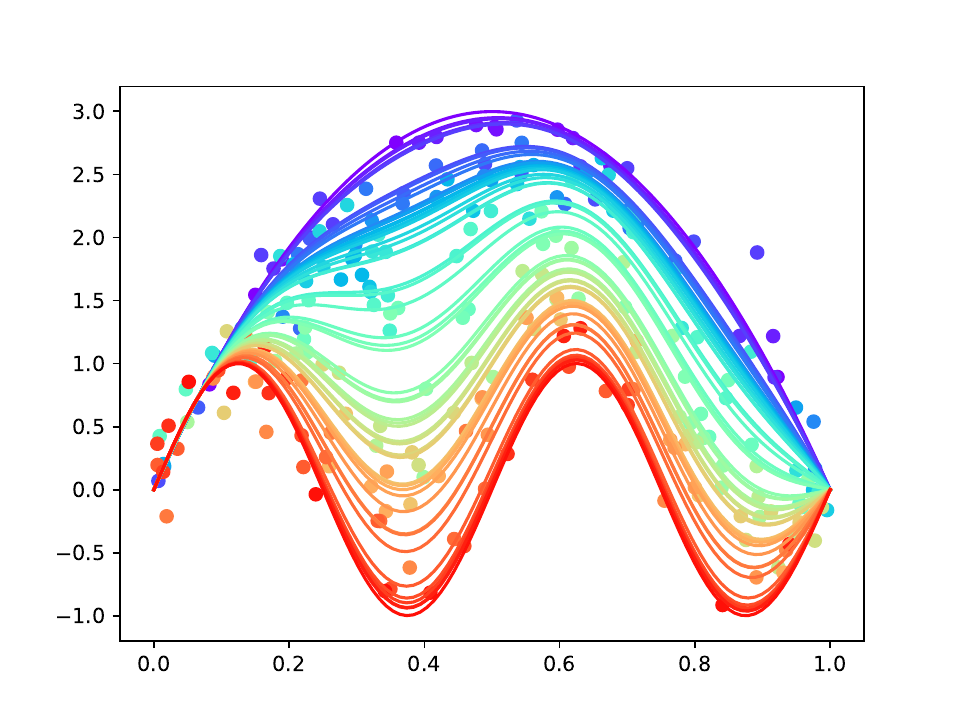}&
    \includegraphics[width=4.5cm]{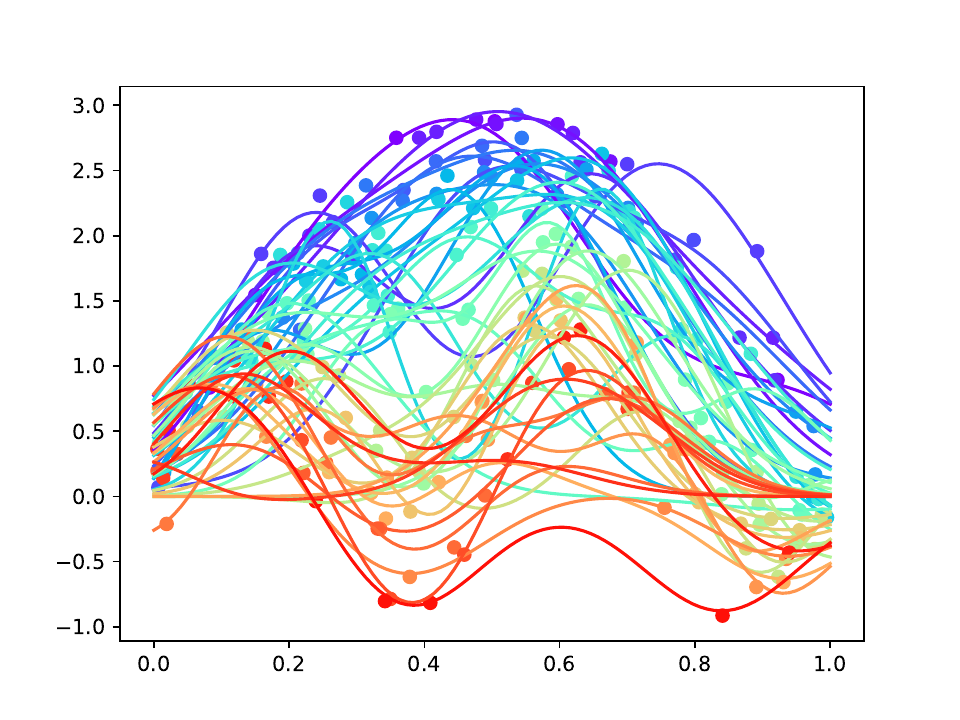} \\
    (a) Training samples& (b) Ground truth& (c) GP\\
    \includegraphics[width=4.5cm]{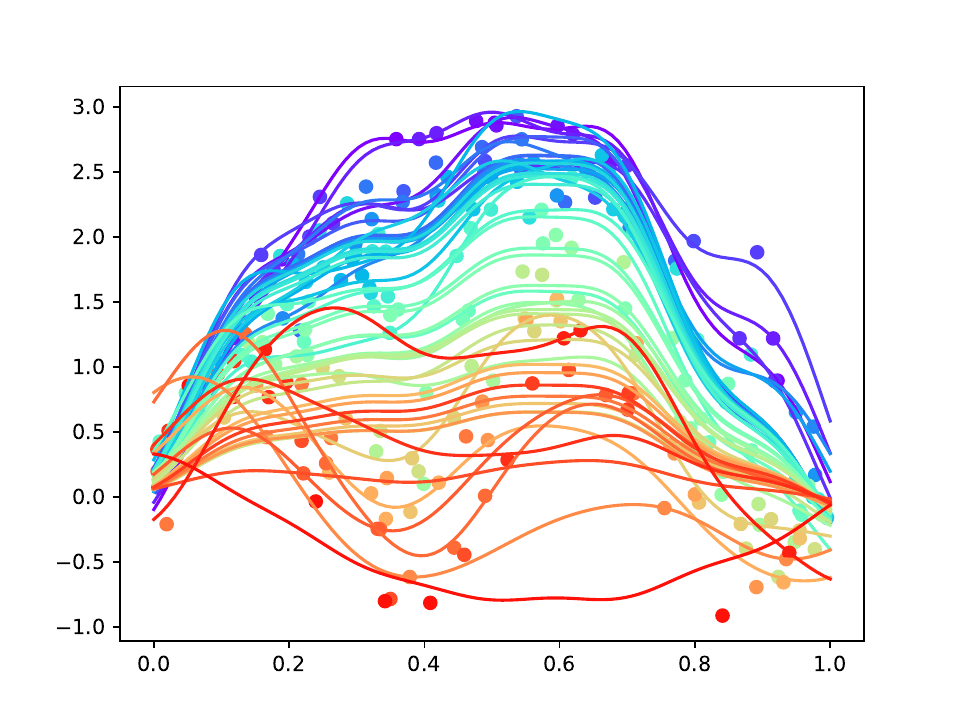}&
    \includegraphics[width=4.5cm]{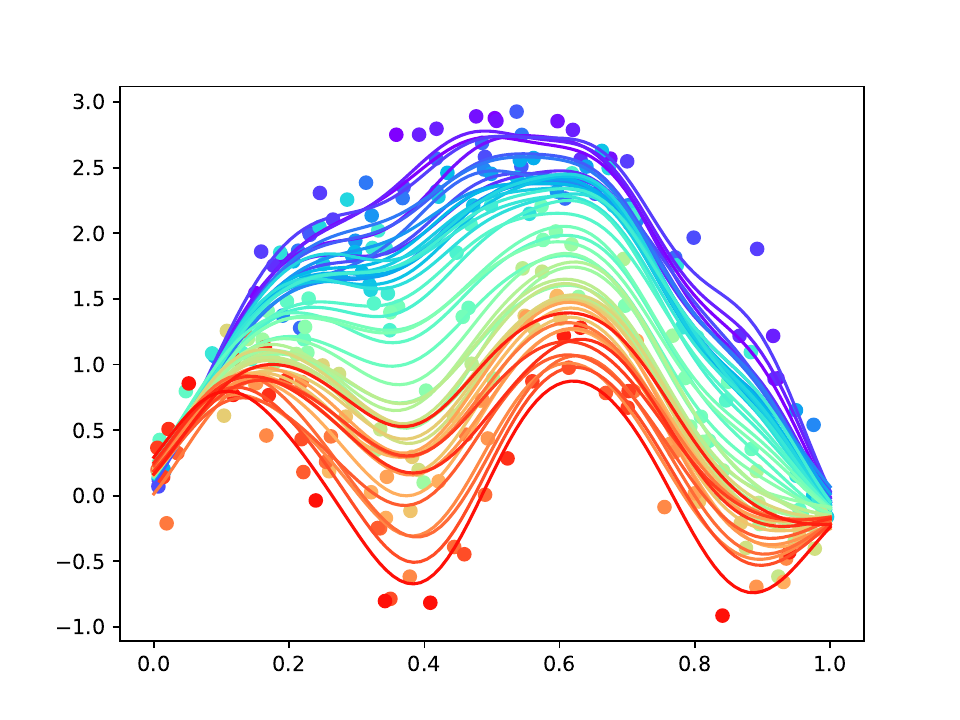}&
    \includegraphics[width=4.5cm]{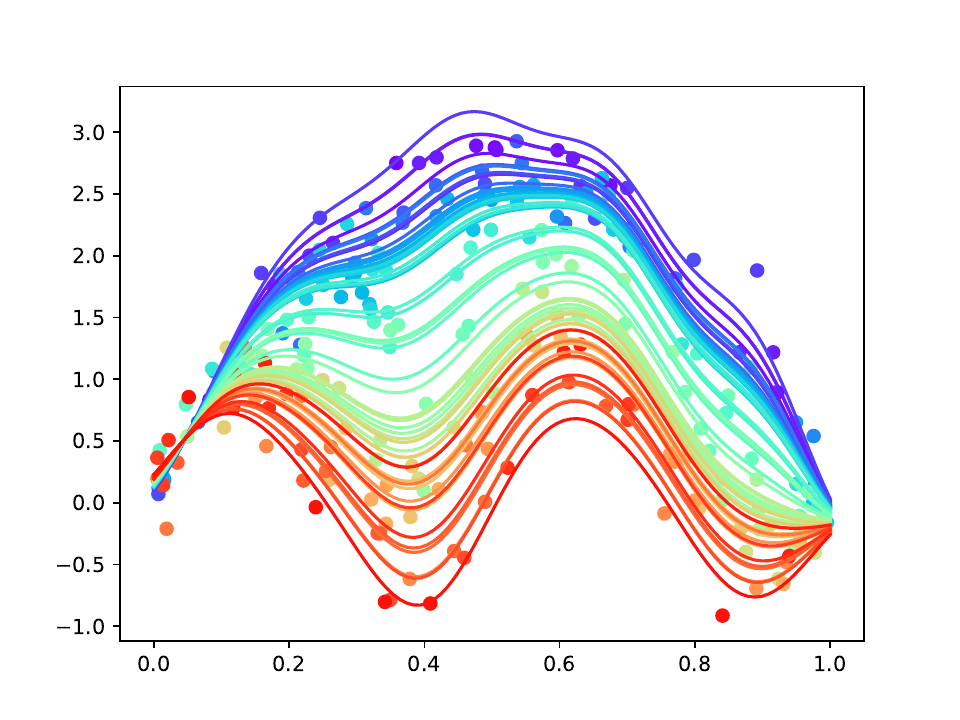} \\
    (d) ICM& (e) HBGP& (f) GP-ePCA
    \end{tabular}
    \caption{Result of artificial dataset using five samples / task for training. The colors of the scatter plot and function indicate a value of $z_i$.}
    \label{fig_fitting}
\end{figure}


\begin{figure}
    \centering
    \begin{tabular}{cc}
    \includegraphics[width=6cm]{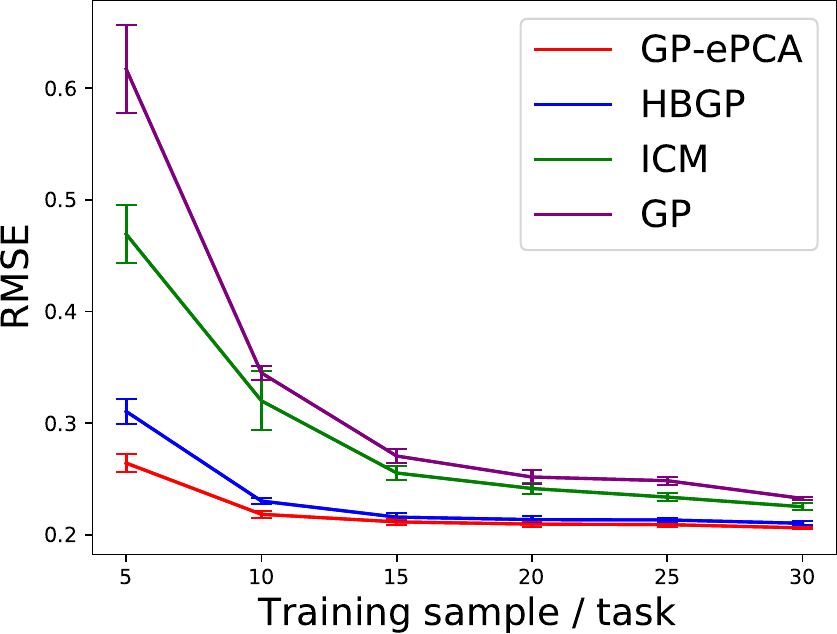}&
    \includegraphics[width=6cm]{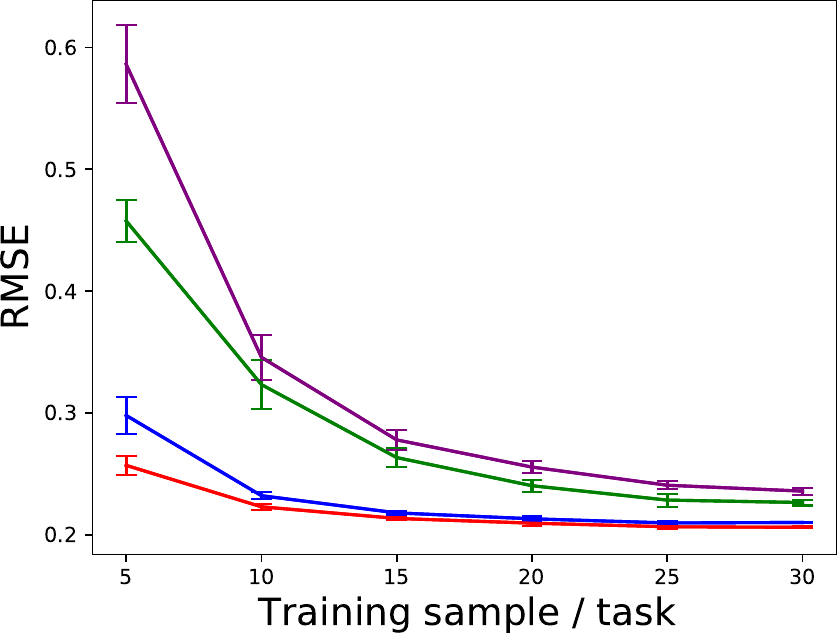} \\
    (a) Training task& (b) Test task \\
    \end{tabular}
    \caption{RMSE of the GP-ePCA and GP for artificial dataset. Line and bar mean average and standard deviation of RMSE, respectively.}
    \label{fig_gene_error_artificial}
\end{figure}

\subsection{The performance of GP-ePCA varying the rank}

Next, we evaluate the impact of the rank of A on RMSEs when a dimension of a set of the true functions differs, or the domain shift between tasks occurs. For this purpose, we use the three artificial datasets. Each artificial dataset is generated by the following equations:
\mbox{}\\
\noindent{\textbf{Artificial data 1:}}\mbox{}
\begin{align*}
    y_{in} &= \sin(2\pi x_{in}) + 3z_i + \varepsilon_{in} \\
    x_{in} &\sim U(0,1)
\end{align*}
\noindent{\textbf{Artificial data 2:}}\mbox{}
\begin{align*}
    y_{in} &= \sin(2\pi (x_{in}+z_i)) +\varepsilon_{in} \\
    x_{in} &\sim U(0,1)
\end{align*}
\noindent{\textbf{Artificial data 3:}}\mbox{}
\begin{align*}
    y_{in} &= \sin(2\pi (x_{in}-z_i)) + 3z_i + \varepsilon_{in} \\
    x_{in} &\sim U(z_i,z_i+1)
\end{align*}
In each dataset, training and test task sizes are set $50$, and training and test data sizes of each task are set $5$. The parameter of a noise is set $\beta\inv = 0.2^2$.

\begin{figure}
    \centering
    \begin{tabular}{ccc}
    \includegraphics[width=4.5cm]{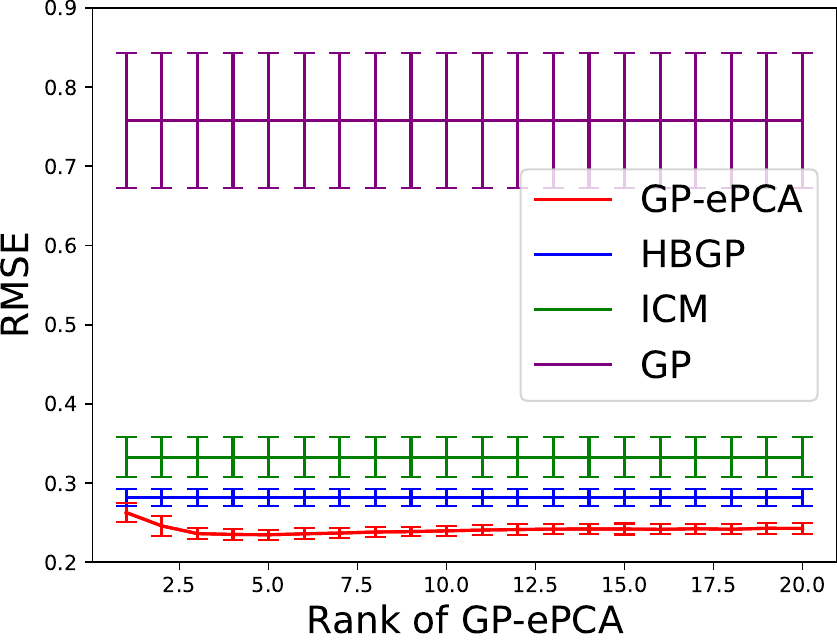}&
    \includegraphics[width=4.5cm]{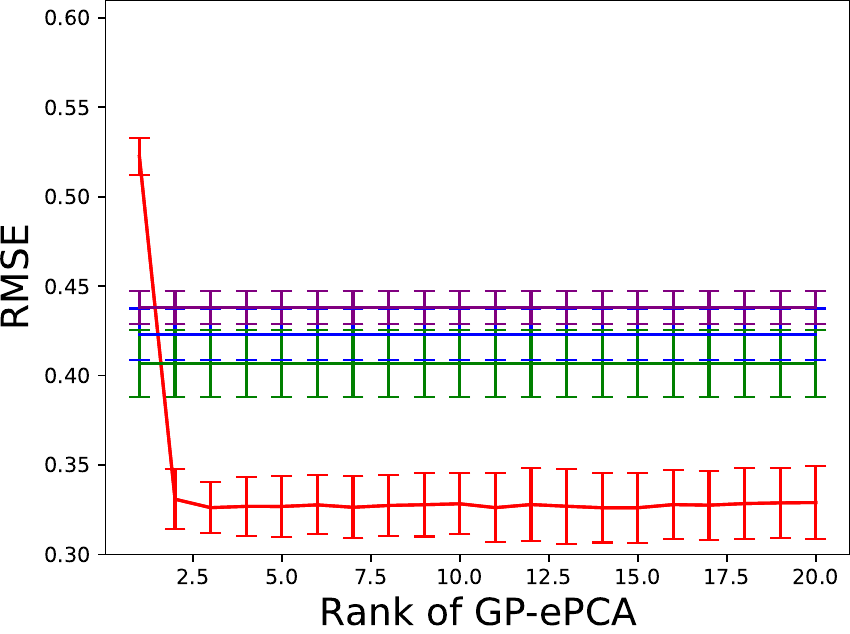}&
    \includegraphics[width=4.5cm]{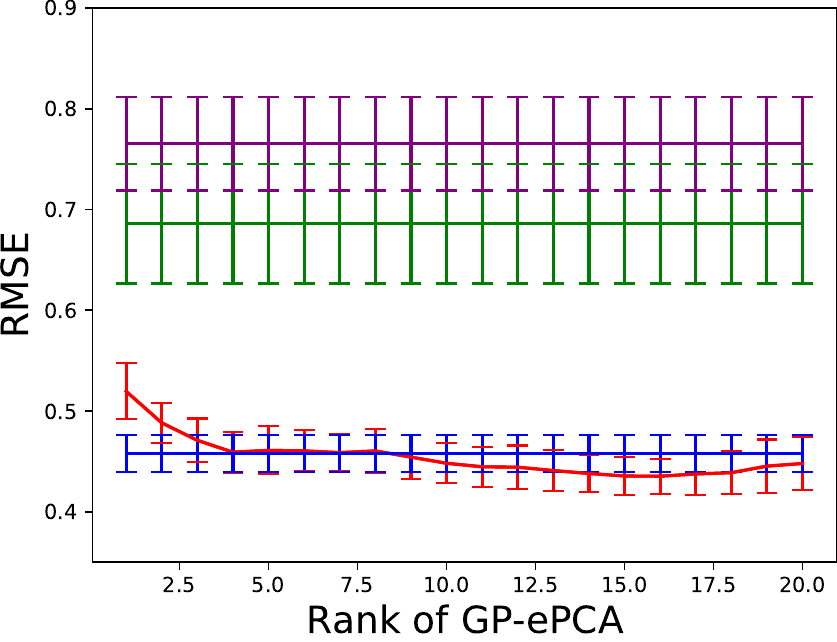} \\
    (a) Artificial data 1& (b) Artificial data 2& (c) Artificial data 3 \\
    \includegraphics[width=4.5cm]{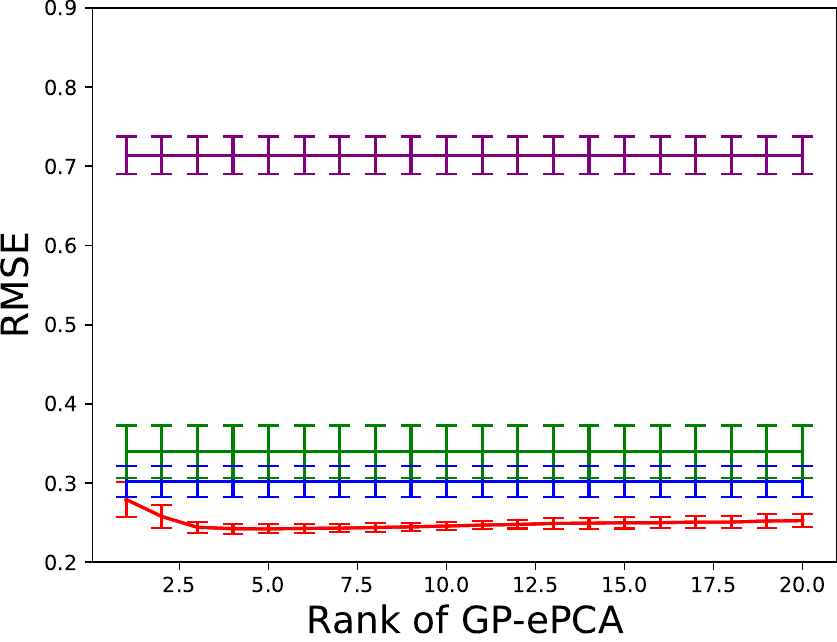}&
    \includegraphics[width=4.5cm]{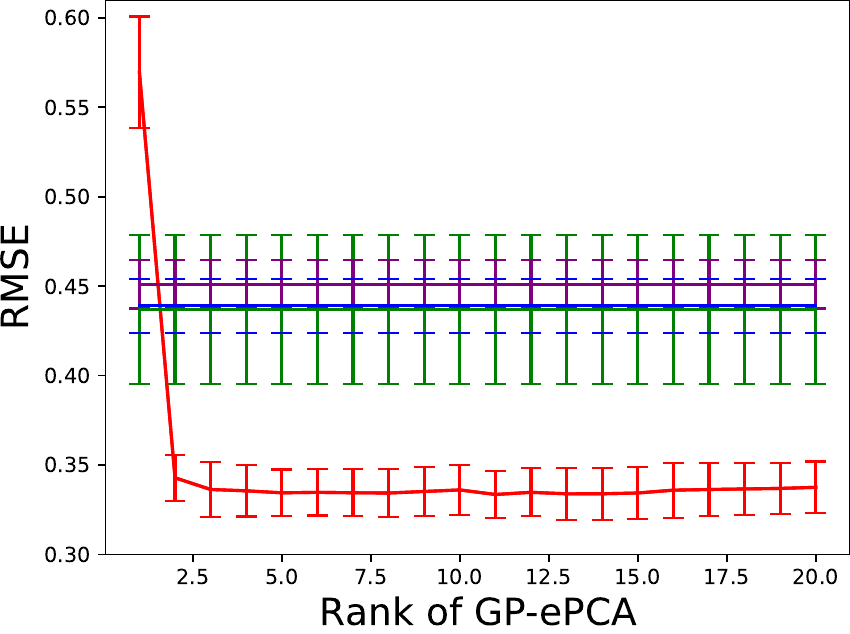}&
    \includegraphics[width=4.5cm]{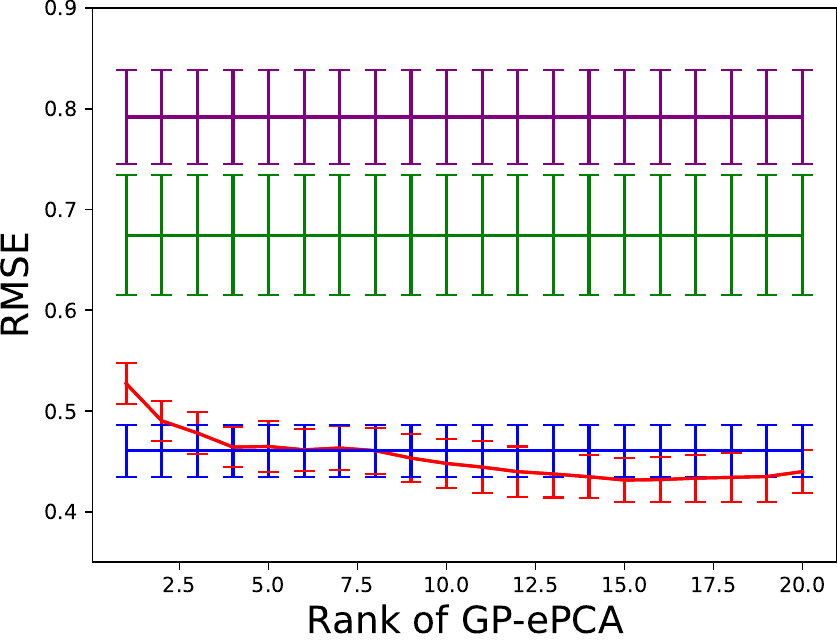} \\
    (d) Artificial data 1& (e) Artificial data 2& (f) Artificial data 3
    \end{tabular}
    \caption{The performance of GP-ePCA varying the rank.(a) -- (c) RMSEs for training tasks. (d) -- (f) RMSEs fo test tasks.}
    \label{fig_dim}
\end{figure}

Figure~\ref{fig_dim} shows a difference of RMSE varying the rank. In Artificial data 1, there is no significant difference in RMSE for any rank. This is because a set of true functions
lies on a one-dimensional subspace on functional space and domain-shift between tasks does not occur. On the other hand, the RMSE is smaller when the rank is greater than 2 in Artificial data 2, since a set of true functions lies on a two-dimensional subspace on functional space. In Artificial data 3, since the domain of each task shifts in conjunction with the function shift, GP posteriors lie on a curve. Therefore, the more the rank is increased, the smaller the RMSE is. From these results, the proposed method tends not to increase the RMSE even when the rank is increased. This is because even if  we increase the rank of the subspace, the RMSE of GP-ePCA becomes about the same as the RMSE of the GP posterior since the input data of GP-ePCA is GP posteriors estimated from datasets.

\subsection{Computational cost of GP-ePCA}
The third experiment is to compare the computational cost of exact GP-ePCA and sparse GP-ePCA by using an artificial data of the first experiment, where the number of inducing points of sparse GP-ePCA is fixed to $20$. We measure the time of each method when the task size is increased and sample size of each task is fixed to $10$. To simplify, neither method estimates hyperparameter.

Figure~\ref{fig_calculation_time} shows the calculation time of each method. While the computational cost of exact GP-ePCA increased fourth power with the task size, that of sparse GP-ePCA increased linearly. For this reason, it is recommended to use sparse GP-ePCA in practice.

\begin{figure}
    \centering
    \includegraphics[width=8cm]{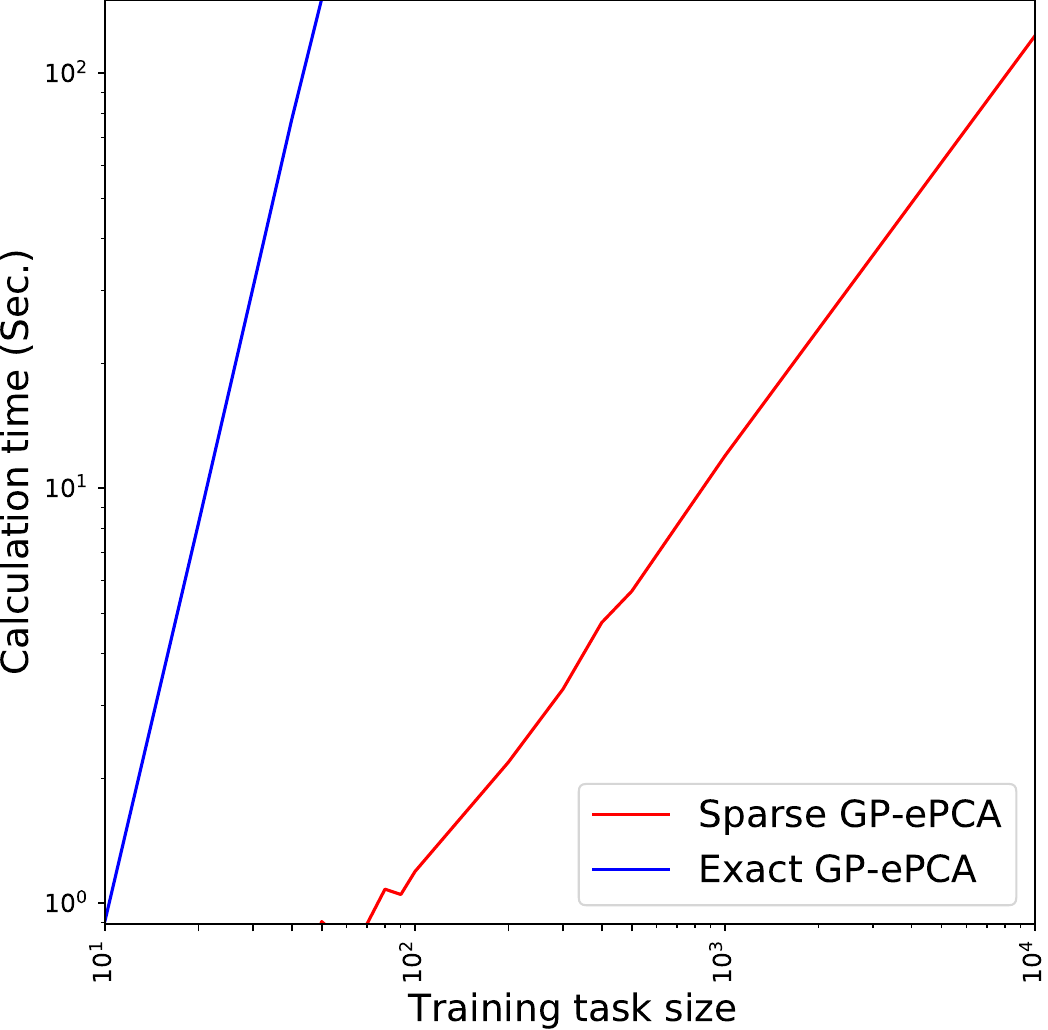}
    \caption{Calculation time of exact GP-ePCA and sparse GP-ePCA.}
    \label{fig_calculation_time}
\end{figure}

\subsection{Evaluation of the performance by using real datasets}

\begin{table}[thb]
  \centering
  \caption{RMSE of each method for real datasets.}
  \scalebox{0.8}{
  \begin{tabular}{l||c|c|c|c}
    \multirow{2}{*}{\diagbox{Methods}{Datasets}} & \multicolumn{2}{c|}{\tt{Computer survey}} & \multicolumn{2}{c}{\tt{Movie lens 100k}} \\ \cline{2-5}
     & Training task & Test task & Training task & Test task \\ \hline \hline
    Single task GP & 3.6665$\pm$0.0941 & 4.2047$\pm$0.0945 & 1.1426$\pm$0.0238 & 1.2718$\pm$0.0467 \\ 
     ICM & 2.0806$\pm$0.0365 & 2.2427$\pm$0.0463 & 1.1206$\pm$0.0182 & 1.2095$\pm$0.0230 \\ 
    HBGP & 2.3082$\pm$0.1067 & 2.2069$\pm$0.0481 & 1.1005$\pm$0.0288 & $\mathbf{1.1416\pm0.0178}$ \\ 
    GP-ePCA(L=1) & 2.1232$\pm$0.0451 & 2.1606$\pm$0.0441 & 1.0847$\pm$0.0241 & 1.1497$\pm$0.0180 \\ 
    GP-ePCA(L=3) & $\mathbf{2.0723\pm0.0382}$ & $\mathbf{2.1605\pm0.0468}$ & $\mathbf{1.0833\pm0.0225}$ & 1.1436$\pm$0.0151 \\ 
    GP-ePCA(L=5) & 2.1037$\pm$0.0546 & 2.1761$\pm$0.0577 & 1.0886$\pm$0.0265 & 1.1422$\pm$0.0177 \\ \hline
    \end{tabular}
  }
  \label{tbl_rmse_real_data}
\end{table}

In this experiment, we demonstrate the effectiveness of the proposed method by using two real datasets: computer survey dataset\footnote{\normalsize https://github.com/probml/pmtk3/tree/master/data/conjointAnalysisComputerBuyers} and movie lens 100k dataset\footnote{\normalsize https://grouplens.org/datasets/movielens/100k/}. The computer survey dataset is a survey data of 190 people who rates the likelihood of purchasing one of 20 personal computers, each of which has 13 features such as price, CPU, and RAM. In this experiment, we consider a task estimating each person's preference for personal computers by regarding each computer's features and each person's rating as input data and output data, respectively. After $100$ tasks are sampled randomly as the training task, $10$ training data are sampled and the rest of the data is used as test data. The rest of the tasks are used for the test task, each of whose dataset divides into $5$ training data and $15$ test data. We calculate the average and standard deviation of RMSE when $T=5$ times iterates.

The movie lens 100k dataset is a rating data of 943 users who rated their preference of 1682 movies. Each movie has a $13$ features (movie id and genres). In this experiment, a task is to predict each user's preference for movies by regarding movie's features and user's ratings are input and output, respectively. After $50$ tasks are sampled randomly as the training task, $10$ training data are sampled and the rest of the data is used as test data. In the test tasks, after $50$ tasks are sampled from the rest of the tasks, $5$ training data are sampled and the rest of the dataset is used for the test data. We calculate the average and standard deviation of RMSE when $T=5$ times iterates.

Table~\ref{tbl_rmse_real_data} shows the RMSE of each method. In the computer survey dataset, the RMSE of GP-ePCA of $L=3$ is the lowest compared to the other methods in both training and test tasks. In particular, although ICM's RMSE is similar to that of the proposed method in training tasks, the RMSE of the proposed method is smaller than that of the ICM in new tasks. In the movie lens 100k dataset, while the RMSE of the proposed method is the smallest of all the methods in training tasks, the RMSE of the HBGP is the smallest of all the methods in test tasks. However, considering the standard deviation, there is no difference between the RMSEs of HBGP and the proposed method, since the discrepancy between the RMSEs of HBGP and the proposed method was within the standard deviation. From these results, our approach is effective for few-shot learning.

\section{Conclusion}
In this study, we proposed a PCA for a set of GP posteriors. Since a structure of a set of GPs is nontrivial, we defined the space of GP posteriors and proved that the space becomes a finite dually flat subspace. Given this fact, PCA for a set of GP posteriors can be regarded as an e-PCA or m-PCA for a set of finite-dimensional multivariate normal distributions. Furthermore, we proposed a fast algorithm, which reduces the calculation order from $\mathcal{O}(IN^3)$ to $\mathcal{O}(Im^3)$, where $N \gg m$. We demonstrated that the proposed algorithm can be applied to multi-task learning and meta-learning.

\subsection*{Acknowledgments}
This work was supported by JSPS KAKENHI grant numbers 17H01793 and 20K19865.

\bibliographystyle{apa}
\bibliography{NECO-21-001-72R1-Reference}

\appendix
\section{Woodbury's matrix inversion Lemma and its derived Lemma}
\label{sec_appendix_woodbury}
\begin{lemma}[\citet{Woodbury1950}]
Let $\mathbf{A}$, $\mathbf{U}$, $\mathbf{B}$, and $\mathbf{V}$ be arbitrary $N \times N$, $N \times M$, $M \times M$, and $M \times N$ matrices, respectively. Suppose that there are inverse matrices of $\mathbf{A}$ and $\mathbf{B}$. Then, the following equation holds.
\begin{align*}
    (\mathbf{A}+\mathbf{U}\mathbf{B}\mathbf{V})\inv = \mathbf{A}\inv-\mathbf{A}\inv\mathbf{U}(\mathbf{B}\inv+\mathbf{V}\mathbf{A}\inv\mathbf{U})\inv\mathbf{V}\mathbf{A}\inv
\end{align*}
\label{lemma_woodbury}
\end{lemma}

\begin{lemma}
In Lemma~\ref{lemma_woodbury}, when $N=M$ and $\vK:=\mathbf{A}=\mathbf{U}=\mathbf{V}$, the following equation holds.
\begin{align*}
    (\mathbf{K}+\mathbf{K}\mathbf{B}\mathbf{K})\inv = \mathbf{K}\inv-(\mathbf{B}\inv+\mathbf{K})\inv
\end{align*}
\label{lemma_woodbury2}
\end{lemma}

\begin{lemma}
Let $\vK_+$, $\vK$, and $\vV$ be arbitrary $M \times N$, $N \times N$, and $N \times N$ matrices, respectively. Suppose that there are inverse matrices of $\vK$ and $\vV$. Then, the following equation holds:
\begin{align*}
    \vSigma_+\vSigma\inv=\vK_+\vK\inv,
\end{align*}
where $\vSigma_+:=\vK_+ + \vK_+\vV\vK$ and $\vSigma:=\vK+\vK\vV\vK$.
\label{lemma_woodbury3}
\end{lemma}
\begin{proof}
\begin{align*}
    \vSigma_+\vSigma\inv=&(\vK_+ + \vK_+\vV\vK)(\vK+\vK\vV\vK)\inv \\
    =&\vK_+(\vI + \vV\vK)(\vI+\vV\vK)\inv\vK\inv \\
    =&\vK_+\vK\inv
\end{align*}
\end{proof}

\begin{lemma}
Let $\vK$ and $\vV$ be $N \times N$ and $N \times N$ non-singular matrices, respectively, and let $\vK_\ast$ and $\vK_{\ast\ast}$ be arbitrary $M \times N$ and $M \times M$ matrices, respectively, including $\vK$ as a sub-matrix, where $M(>N)$. Then, the following equation holds:
\begin{align}
    \vTheta_{\ast\ast}\vK_\ast\vK\inv\vTheta\inv = \vK_{\ast\ast}\inv\vK_\ast,
\end{align}
where $\vTheta_{\ast\ast}:=(\vK_{\ast\ast} + \vK_\ast\vV\vK_\ast)\inv$, $\vTheta:=(\vK+\vK\vV\vK)\inv$.
\label{lemma_woodbury4}
\end{lemma}
\begin{proof}
\begin{align*}
    &(\vK_{\ast\ast}+\vK_\ast\vV{\vK_\ast}\T)\inv\vK_\ast\vK\inv(\vK+\vK\vV\vK) \\
    =&(\vK_{\ast\ast}\inv\vK_\ast+\vK_{\ast\ast}\inv\vK_\ast(\vV\inv+\vK)\inv\vK_\ast\T\vK_{\ast\ast}\inv\vK_\ast)\vK\inv(\vK+\vK\vV\vK) \\
    =&(\vK_{\ast\ast}\inv\vK_\ast+\vK_{\ast\ast}\inv\vK_\ast(\vV\inv+\vK)\inv\vK)\vK\inv(\vK+\vK\vV\vK) \\
    =&\vK_{\ast\ast}\inv\vK_\ast(\vI+(\vV\inv+\vK)\inv\vK)\vK\inv\{(\vK+\vK\vV\vK)\inv\}\inv \\
    =&\vK_{\ast\ast}\inv\vK_\ast(\vK\inv+(\vV\inv+\vK)\inv)(\vK\inv+(\vV\inv+\vK)\inv)\inv \\
    =&\vK_{\ast\ast}\inv\vK_\ast
\end{align*}
\end{proof}

\section{Proof of Lemmas}
\label{sec_appendix_proof_lemmas}

\subsection{Proof of Lemma~\ref{lemma_e_flatness}}
For any $\vrho=\{\vmu,\vSigma\}$, the natural parameter $\vxi^\ast(\vrho)=(\vtheta_\ast,{\rm vec}(\vTheta_{\ast\ast})) \in \cT^\ast$ represented as follows:
\begin{align*}
    \vtheta_\ast &= \vSigma\inv_{\ast\ast}\vmu_\ast, \\
    \vTheta_{\ast\ast} &= \vSigma\inv_{\ast\ast},
\end{align*}
where $\vmu_\ast=\vmu_{\ast 0}+\vK_\ast\vK\inv(\vmu-\vmu_0)$ and $\vSigma_{\ast\ast}=\vK_{\ast\ast}+\vK_\ast\vK\inv(\vSigma-\vK)\vK\inv\vK\T_\ast$\footnote{\normalsize Since a coefficient of e-coordinate is irrelevant to the proof of the lemma, we abbreviate the coefficient of the e-coordinate.}. Similarly, $\vxi(\vrho)=(\vtheta,{\rm vec}(\vTheta)) \in \cT$ is described as
\begin{align*}
    \vtheta &= \vSigma\inv\vmu, \\
    \vTheta &= \vSigma\inv.
\end{align*}
From Lemma~\ref{lemma_woodbury} and Lemma~\ref{lemma_woodbury2} in Appendix~\ref{sec_appendix_woodbury}, we have
\begin{align}
    \vTheta_{\ast\ast}&=(\vK_{\ast\ast}+\vK_\ast\vK\inv(\vTheta\inv-\vK)\vK\inv{\vK_\ast}^T)\inv, \notag \\
    &=\vK_{\ast\ast}\inv-\vK_{\ast\ast}\inv\vK_\ast\vK\inv((\vTheta\inv-\vK)\inv+\vK\inv)\inv\vK\inv{\vK_\ast}^T\vK_{\ast\ast}\inv \notag \\
    &=\vK_{\ast\ast}\inv-\vK_{\ast\ast}\inv\vK_\ast(\vK\inv-\vTheta){\vK_\ast}^T\vK_{\ast\ast}\inv. \label{eq_affine_theta1}
\end{align}
Letting $\vV:=\vK\inv(\vSigma-\vK)\vK\inv$ and $\vV_{\ast\ast}:=\left[\begin{array}{cc}\vV & \vzero \\\vzero & \vzero\end{array}\right]$, $\vtheta_\ast$ is described as follows:
\begin{align*}
    \vtheta_\ast &= \vTheta_{\ast\ast}(\vmu_{\ast 0}+\vK_\ast\vK\inv(\vmu-\vmu_0)) \\
    &= \vTheta_{\ast\ast}(\vmu_{\ast 0}+\vK_\ast\vK\inv(\vmu-\vmu_0-\vK\vV\vmu_0+\vK\vV\vmu_0)) \\
    &= \vTheta_{\ast\ast}((\vI_{\ast\ast}+\vK_{\ast\ast}\vV_{\ast\ast})\vmu_{\ast 0}+\vK_\ast\vK\inv(\vmu-(\vI+\vK\vV)\vmu_0)) \\
    &= \vTheta_{\ast\ast}((\vK_{\ast\ast}+\vK_{\ast\ast}\vV_{\ast\ast}\vK_{\ast\ast})\vK\inv_{\ast\ast}\vmu_{\ast 0}+\vK_\ast\vK\inv(\vmu-(\vK+\vK\vV\vK)\vK\inv\vmu_0))
\end{align*}
Since $\vTheta\inv_{\ast\ast}=\vK_{\ast\ast}+\vK_{\ast\ast}\vV_{\ast\ast}\vK_{\ast\ast}=\vK_{\ast\ast}+\vK_\ast\vV\vK\T_\ast$ and $\vTheta\inv=\vK+\vK\vV\vK\T$, we have
\begin{align*}
    \vtheta_\ast&= \vTheta_{\ast\ast}(\vTheta\inv_{\ast\ast}\vK\inv_{\ast\ast}\vmu_{\ast 0}+\vK_\ast\vK\inv(\vTheta\inv\vtheta-\vTheta\inv\vK\inv\vmu_0)) \\
    &= \vK\inv_{\ast\ast}\vmu_{\ast 0}+\vTheta_{\ast\ast}\vK_\ast\vK\inv\vTheta\inv(\vtheta-\vK\inv\vmu_0).
\end{align*}
By using Lemma~\ref{lemma_woodbury4} in Appendix~\ref{sec_appendix_woodbury}, the following equation holds:
\begin{equation}
    \vtheta_\ast= \vK\inv_{\ast\ast}\vmu_{\ast 0}+\vK\inv_{\ast\ast}\vK_\ast(\vtheta-\vK\inv\vmu_0).
    \label{eq_affine_theta2}
\end{equation}

Since $\vtheta_\ast$ and $\vTheta_{\ast\ast}$ are transformed from $\vtheta$ and $\vTheta$ by Eqs.~\eqref{eq_affine_theta1} and \eqref{eq_affine_theta2}, lemma~\ref{lemma_e_flatness} can be proved.

\subsection{Proof of Lemma~\ref{lemma_equivalent_of_kl_divergence}}
From the definition of $\cT^\ast$, we have $q(\vf_\ast\mid\vrho)=p(\vf_+\mid\vf)p(\vf\mid\vrho)$. Therefore, the KL divergence $\KL{q(\vf_\ast\mid\vrho)}{q(\vf_\ast\mid\vrho')}$ can be decomposed as follows. 
\begin{align}
    \KL{q(\vf_\ast\mid\vrho)}{q(\vf_\ast\mid\vrho')}=&\KL{q(\vf\mid\vrho)}{q(\vf\mid \vrho')} \nonumber \\
    &+\mathbb{E}_{q(\vf\mid\vrho)}[\KL{p(\vf_+\mid\vf)}{p(\vf_+\mid\vf)}]
\label{KL_divergence}
\end{align}
The above equation entails that the second term of Eq.~\eqref{KL_divergence} is zero.
Hence, lemma~\ref{lemma_equivalent_of_kl_divergence} holds.

\subsection{Proof of Lemma~\ref{lemma_equivalent_gppca}}
Since $\cT^\ast$ is dually flat, we can take a dual coordinate system in $\cS^\ast$ such that e-coordinate is decomposed as $\vxi^\ast=({\vxi^\ast\coI}\T,{\vxi^\ast\coII}\T)\T$ and m-coordinate is decomposed as $\vzeta^\ast=({\vzeta^\ast\coI}\T, {\vzeta^\ast\coII}\T)\T$, and $\cT^\ast$ is as a subspace defined in the m-coordinate $\{\vzeta^\ast\mid \vzeta^\ast\coII=\mathbf{0}\}$. 

Let $\cM^\ast_e\subset \cT^\ast$ be the $L$-dimensional e-flat submanifold minimizing Eq.~\eqref{eq_obj_func_GP_PCA_ast} for $P$. Since the case $L=K$ is trivial, we assume $L<K$. Let $\vzeta^\ast_i$ be the m-coordinate of $p(\vf^\ast\mid \vrho_i)\in P$ and $\hat{\vzeta}^\ast_i$ be the m-projection of $p(\vf^\ast\mid \vrho_i)\in P$ onto $\cM^\ast_e$. By using the basis vectors $\vU^\ast=(\vu^\ast_0,\vu^\ast_1,\vu^\ast_2,\ldots,\vu^\ast_I)\T$, $\hat{\vzeta}^\ast_i$ is represented in the e-coordinate as $\hat{\vxi}^\ast_i = (1,{\vw^\ast_i}\T)\vU^\ast$.

The derivative of Eq.~\eqref{eq_obj_func_GP_PCA_ast} with respect to the parameters $\vW^\ast$ and $\vU^\ast$ is given by
\begin{align}
    \frac{\partial E^\ast(\vW^\ast,\vU^\ast)}{\partial \vW^\ast}&=(\hat{\vZeta}^\ast-\vZeta^\ast)\tilde{\vU}^{\ast{\rm T}}, \label{eq:dLdw} \\
    \frac{\partial E^\ast(\vW^\ast,\vU^\ast)}{\partial \vU^\ast}&=\vW^{\ast{\rm T}}({\hat{\vZeta}}^\ast-\vZeta^\ast), \label{eq:dLdu}
\end{align}
where $\tilde{\vU}^\ast=(\vu^\ast_1,\vu^\ast_2,\ldots,\vu^\ast_L)\T$. We consider that $\vu^\ast$ is decomposed as $\vu^\ast=({\vu^\ast\coI}\T,{\vu^\ast\coII}\T)\T$, and let $\tilde{\vU}^\ast\coI$, $\vZeta^\ast\coI$ and $\hat{\vZeta}^\ast\coI$ be matrices of $\{\vu^\ast_{\mathrm{I},l}\}^L_{l=1}$, $\{\vzeta^\ast_{\mathrm{I},i}\}^I_{i=1}$ and $\{\hat{\vzeta}^\ast_{\mathrm{I},i}\}^I_{i=1}$, respectively. Since $\cM^\ast_e$ is a stationary point of the GP-ePCA$(\cS^\ast)$ constrained in $\cT^\ast$, (i.e.,
from Eqs.~\eqref{eq:dLdw} and \eqref{eq:dLdu}), we have
\begin{equation}
    (\hat{\vZeta}^\ast\coI - \vZeta^\ast\coI){\tilde{\vU}^\ast\coI} = \vzero,
\end{equation}
\begin{equation}
    \vW^\ast (\hat{\vZeta}^\ast\coI - \vZeta^\ast\coI) = \vzero
\end{equation}
Further, since $\vZeta^\ast$ and $\hat{\vZeta}^\ast$ are included in $\cM^\ast_e$, it holds that $\hat{\vZeta}^\ast\coII=\vZeta^\ast\coII=\vzero$.
Therefore (\ref{eq:dLdw}) and (\ref{eq:dLdu}) are all zeros.
That means $\cM^\ast_e$ is a stationary point of the GP-ePCA$(\cS^\ast)$ in $\cS^\ast$ as well,
which proves the Lemma.

\subsection{Proof of Lemma~\ref{lemma_m_flatness}}
\begin{proof}
In the case of a Gaussian distribution parameterized $\vmu$ and $\vSigma$, the natural parameters represented using $\vmu$ and $\vSigma$ are as follows:
\begin{align*}
    \veta_\ast &= \vmu_\ast, \\
    \vEta_\ast &= \vSigma_{\ast\ast}+\vmu_\ast\vmu_\ast\T,
\end{align*}
Then, the following equations hold:
\begin{align*}
    \veta_\ast=&\vmu_\ast \\
    =&\vmu_{\ast0}+\vK_\ast\vK\inv(\vmu-\vmu_0) \\
    =&\vmu_{\ast0}+\vK_\ast\vK\inv(\veta-\vmu_0),
\end{align*}
and
\begin{align*}
    \vEta_\ast =& \vK_{\ast\ast}+\vK_\ast\vK\inv(\vSigma-\vK)\vK\inv\vK\T_\ast+\veta_\ast\veta\T_\ast \\
    =& \vK_{\ast\ast}+\vK_\ast\vK\inv(\vEta-\veta\veta\T-\vK)\vK\inv\vK\T_\ast+\veta_\ast\veta\T_\ast. \\
    =& \vK_{\ast\ast}+\vK_\ast\vK\inv(\vEta-\vK)\vK\inv\vK\T_\ast-\vK_\ast\vK\inv\veta\veta\T\vK\inv\vK\T_\ast \notag \\
    &+(\vmu_{\ast0}+\vK_\ast\vK\inv(\veta-\vmu_0))(\vmu_{\ast0}+\vK_\ast\vK\inv(\veta-\vmu_0))\T. \\
    =& \vK_{\ast\ast}+\vK_\ast\vK\inv(\vEta-\vK-\veta\veta\T)\vK\inv\vK\T_\ast \notag \\
    &+\vmu_{\ast0}\vmu_{\ast0}\T+\vmu_{\ast0}(\veta-\vmu_0)\T\vK\inv\vK\T_\ast+\vK_\ast\vK\inv(\veta-\vmu_0)\vmu\T_{\ast0} \notag \\
    &+\vK_\ast\vK\inv(\veta-\vmu_0)(\veta-\vmu_0)\T\vK\inv\vK\T_\ast. \\
    =& \vK_{\ast\ast}+\vK_\ast\vK\inv(\vEta-\vK-\veta\veta\T)\vK\inv\vK\T_\ast \notag \\
    &+\vmu_{\ast0}\vmu_{\ast0}\T+\vmu_{\ast0}(\veta-\vmu_0)\T\vK\inv\vK\T_\ast+\vK_\ast\vK\inv(\veta-\vmu_0)\vmu\T_{\ast0} \notag \\
    &+\vK_\ast\vK\inv(\veta\veta\T-\vmu_0\veta\T-\veta\vmu\T_0+\vmu_0\vmu\T_0)\vK\inv\vK\T_\ast. \\
    =& \vK_{\ast\ast}+\vK_\ast\vK\inv(\vEta-\vK-\vmu_0\veta\T-\veta\vmu\T_0+\vmu_0\vmu\T_0)\vK\inv\vK\T_\ast \notag \\
    &+\vmu_{\ast0}\vmu_{\ast0}\T+\vmu_{\ast0}(\veta-\vmu_0)\T\vK\inv\vK\T_\ast+\vK_\ast\vK\inv(\veta-\vmu_0)\vmu\T_{\ast0}.
\end{align*}
From the above, we show that the Lemma~\ref{lemma_m_flatness} is proved.
\end{proof}

\end{document}